\documentclass[11pt]{article}
\usepackage{microtype}
\usepackage{graphicx}
\usepackage{subcaption}
\usepackage{booktabs} 

\usepackage{amsmath, amssymb}
\usepackage{amsthm}
\usepackage{amsfonts}

\usepackage{amssymb}
\usepackage{appendix}
\usepackage{enumitem}
\usepackage{bm}
\RequirePackage{fullpage}
\RequirePackage{algorithm}
\RequirePackage{algorithmic}

\newcommand{\arm}{\mathtt{arm}}
\newcommand{\agent}{\mathtt{agent}}
\DeclareMathOperator*{\argmax}{\arg\,\max}

\DeclareMathOperator*{\hc}{{:}}
\newcommand{\rem}[1]{\textbf{\textit{#1}}}

\newtheorem{theorem}{Theorem}
\newtheorem{lemma}{Lemma}
\newtheorem{definition}{Definition}
\newtheorem{claim}{Claim}
\newtheorem{proposition}{Proposition}

\newcommand{\mytilde}{\raise.17ex\hbox{$\scriptstyle\mathtt{\sim}$}}
\newcommand{\name}{UCB-D3 with Local Deletion}
\newcommand{\mainAlg}{\textsc{UCB-D4}}
\newcommand{\phase}{i}
\newcommand{\UCB}{\mathtt{UCB}}
\newcommand{\ucon}{\mathtt{UnqC}}

\makeatletter
\g@addto@macro\bfseries{\boldmath}
\makeatother

\title{\bf Beyond $\log^2(T)$ Regret for Decentralized Bandits in Matching Markets}

\author{Soumya Basu\thanks{Google. email: basusoumya@google.com}
\and Karthik Abinav Sankararaman\thanks{Facebook. email: karthikabinavs@gmail.com}
\and Abishek Sankararaman\thanks{Amazon. email: abishek.90@gmail.com. Work done outside of Amazon}}
\date{}

\begin{document}

\maketitle

\begin{abstract}
    We design decentralized algorithms for regret minimization in two sided matching markets with one-sided bandit feedback that significantly improves upon the prior works \cite{liu2020competing,sankararaman2020dominate,decentralized_jordan}. First, for general markets, for any $\varepsilon \mathtt{>} 0$, we design an algorithm that achieves a $O(\log^{1+\varepsilon}(T))$ regret to the agent-optimal stable matching, with unknown time horizon $T$,  improving upon the $O(\log^{2}(T))$ regret achieved in \cite{decentralized_jordan}. Second, we provide the optimal $\Theta(\log(T))$  regret for markets satisfying {\em uniqueness consistency} -- markets where leaving participants don't alter the original stable matching. Previously, $\Theta(\log(T))$ regret was achievable \cite{sankararaman2020dominate,decentralized_jordan} in the much restricted {\em serial dictatorship} setting, when all arms have the same preference over the agents. We propose a phase based algorithm, where in each phase, besides deleting the globally communicated dominated arms, the agents locally delete arms with which they collide often. This \emph{local deletion} is pivotal in breaking deadlocks arising from rank heterogeneity of agents across arms.  We further demonstrate superiority of our algorithm over existing works through simulations.
\end{abstract}

\section{Introduction}
Decentralized decision making by competing agents under uncertainty, each one motivated by one's own objective, is a key feature in online market places, e.g. TaskRabbit, UpWork, DoorDash. An emerging line of research~\cite{aridor2020competing, johari_1, liu2020competing, sankararaman2020dominate, decentralized_jordan} in the field of multi-agent bandits is dedicated to understanding algorithmic principles in the interplay of competition, learning and regret minimization. The two-sided matching market~\cite{gale_shapely} is one such thread, where regret minimization is first studied in \cite{liu2020competing} with a centralized arbiter, and in \cite{sankararaman2020dominate, decentralized_jordan} at different levels of decentralization. 

We study the fully decentralized two-sided matching market, comprising of $N$ demand-side entities (a.k.a. agents) and $K$ supply-side entities (a.k.a. arms). Each entity has a separate preference ranking of the opposite side agents, i.e. each agent over $K$ arms, and each arm over $N$ agents, and aims to match with the most preferred entity.  In the bandit setting, the agents lack prior knowledge of their respective preferences, thus need to learn the preferences only through their {\em own past interactions}, while the arms know their preferences. In each round, every agent simultaneously chooses an arm of their choice, and are either \emph{matched} to their arm of choice and receive a stochastic reward (reward defines agents' preference); or are \emph{blocked}, are notified of this, and receive no reward. 

A matching between agents and arms is \emph{stable} if there exist no pair of agents and arms that are not matched with each other under the matching, prefer each other over their current partners. When each participant knows its preference and the agents propose to the arms across multiple rounds, the system admits any \emph{stable matching} as a Nash equilibrium. Further, the {\em agent-optimal} stable matching is the one, which yields the highest reward among multiple possible stable matchings, to all agents \cite{gale_shapely}. Our objective is to design a uniform protocol for the agents, which allows each agent to quickly find and match with its agent-optimal arm, thus maximizing their cumulative reward.

Under a restrictive special case, known as serial dictatorship, where all arms have the same preference \cite{sankararaman2020dominate} shows it is possible for each agent to attain  $O(NK\log(T))$ cumulative regret (the gap between optimal and achieved reward) in $T$ rounds. They rely on  an Upper confidence bound (UCB) based algorithm, called UCB-D3, that uses strategic deletion of arms done through global communication alongside UCB based explore-exploit. When the decentralization is relaxed, and each agent observes the response of all the arms per round, \cite{decentralized_jordan} designs Collision Avoidance UCB (CA-UCB), which avoids collision by arm deletion and reduced switching of arms. This provides $O(\exp(N^4) \log^2(T))$ regret for each agent with respect to an agent-pessimal stable matching.  Under agent-pessimal stable matching, all the agents obtain the minimum reward among all stable matching, thus it is less desirable than the agent-optimal stable matching. Moreover, the relaxation of decentralization comes at the cost of privacy, and the possible lack of truthfulness of agents which can hurt the performance, and discourage participation.  

Our first contribution is to show through a simple phase based explore then commit protocol each agent can achieve $O(K\log^{1+\varepsilon}(T))$ regret for any $\varepsilon > 0$, for sufficiently large time horizon $T$. We further focus on two-sided markets that satisfy {\em uniqueness consistency} ($\ucon$)~\cite{karpov2019necessary}, where the stable match is \emph{robust} \textemdash namely if any subset of matched pairs of a stable match leave, the remaining matches are still stable. This also relaxes the stringent requirement of homogeneous preferences imposed by the serial dictatorship model. Arguably, robustness to matches departing the system and heterogeneity of preferences is a desirable property in large dynamic markets such as TaskRabbit, InstaCart etc. Under $\ucon$, where UCB-D3 fails due to heterogeneous arm ranking, we design UCB-D4 (UCB-D3 augmented with local deletions) that allows each agent to achieve $O(NK\log(T))$ regret. In local deletion, an agent aggressively and  locally removes the arms to which it collides above a well designed threshold. 
\subsection{Main Contributions}
 \textbf{1. General  markets.} The best existing regret bound for general stable matching due to \cite{decentralized_jordan} obtains a $O(\exp(N^4)\log^2(T))$ {\em agent-pessimal} regret under {\em partial decentralized} feedback in time horizon $T$. The CA-UCB algorithm is used where agents switch with very low frequency to the arm with highest UCB index in a carefully chosen subset. The slow switching allows for collision avoidance but at a cost of high regret. Therefore, we design, phased ETC, a simple phase based algorithm, used uniformly by agents, which sets up a protocol that allows collision free exploration at the beginning of each phase. Then using the Gale-Shapley~\cite{gale_shapely} algorithm commits to a stable matching with their individual estimated preference lists. We prove this algorithm achieves a $O(K \log^{1+\varepsilon}(T))$ {\em agent-optimal} regret under {\em fully decentralized} feedback for any $\varepsilon > 0$, by setting the exploration duration  based on $\varepsilon$. Although our proposed algorithm beats the state-of-the-art CA-UCB guarantees in many dimensions -- fully\,vs\,partially decentralized, agent optimal\,vs\,pessimal regret, $O(K \log^{1+\varepsilon}(T))$\,vs\,$O(\exp(N^4)\log^2(T))$ -- it suffers from cold-start, i.e. it works for $T = \Omega(\exp(N/\Delta^2))$,  where $\Delta$ is the minimum reward gap across all arms and agents. This leaves open the quest for an optimal $O(\log(T))$ regret without the curse of cold-start.

\textbf{2. Markets with uniqueness consistency.} We next focus on  markets with uniqueness consistency (or $\ucon$ in short), where there is a unique and robust stable matching. In this setting, the best known result (excluding  serial dictatorship) is $O(\exp(N^4)\log^2(T))$, achieved by CA-UCB. For serial dictatorship, a special case of $\ucon$ where all arms have the same preference order, the best result, $O((j-1)K\log(T))$ for agent ranked $j$ for all arms, is obtained by UCB-D3. UCB-D3, a phase based algorithm however, cannot incur sub-linear regret when the preferences are heterogeneous (empirically shown in the Section~\ref{sec:experiment}). We introduce UCB-D4, a generalization of UCB-D3 that handles heterogeneous arm preferences under uniqueness consistency, and achieves the coveted  $O(\log(T))$ regret for any $T = \Omega( N/\Delta^2)$, i.e. without the cold-start problem. Our key algorithmic and theoretical insights behind this result are as follows. 


\rem{Algorithmic.} UCB-D4 augments UCB-D3 with a \emph{local deletion}, where an agent in each phase deletes an arm locally if it experiences collision more than a $(\beta\mathtt{\times}\text{phase length})$ times, for an appropriate  $\beta >0$. 
The local deletion plays an important role in eliminating deadlocks that can be created under UCB-D3.  In particular, consider the case when uniqueness consistency holds. Here arms have heterogeneous preferences, and agent $a$ and agent $b$ can block each other from exploring their non-stable matched arms. Interestingly, we discover that such deadlocks do not occur when playing the stable matched arms. 
Due to the specific nature of the deadlock, if an agent `frequently' collides with an arm it is safe to delete that arm locally. The key is to carefully set the threshold of local deletion. A small threshold can remove the stable matched arm with constant probability due to the stochastic feedback, thus incur linear regret. Whereas, a large threshold deter the agent from getting out of the deadlock fast enough to achieve $O(\log(T))$ regret. UCB-D4 with $\beta < 1/K$ strikes the correct balance.

\rem{Theoretical.}  We introduce a dual induction proof-technique linked with $\alpha$-condition. The $\alpha$-condition bestows two orders: one among agents -- left order, and one among arms -- right order. In the left order, the agents have their stable matched arm as best arm once the stable matched arms for higher order agents are removed. For the right order same holds with  the roles of agents and arms swapped. Our dual induction uses these two orders. We show local deletion ensures  the arms inductively, following the right order, become `available' for their respective stable matched agents. This allows the agents to inductively, following the left order, identify and broadcast their respective stable matched arm. The second induction, depends on first, and is driven by global deletion of dominated arms, alongside deadlock resolution due to local deletion.

\textbf{3. Empirical.} We compare the performance of our algorithms against CA-UCB, an algorithm for matching markets under \emph{partial decentralization}: the response of all arms per round is visible. Our proposed phased ETC and $\mainAlg$ both work under {\em full decentralization}: only sees response of the proposed arm. Extensive simulations show, despite the restrictive feedback, phased ETC outperforms CA-UCB in general instances, while $\mainAlg$ does so under uniqueness consistency. We also empirically validate that UCB-D3 produces linear regret under uniqueness consistency.

\subsection{Related Work}
The field of bandit learning has a vast literature with multiple textbooks on this subject \cite{lattimore2020bandit,bubeck2012regret,slivkins2019introduction}. Our work falls in the multi-agent bandits setting, which is effective in modelling applications like wireless networks \cite{wireless_mab, wireless_mab_2}, online advertising \cite{mab_ads}, data-centers and search scheduling \cite{social_learning} where multiple decision makers interact causing an interesting interplay of competition, learning and consensus \cite{bandit_nash, bandit_nash_2}. The study of multi-agent bandits in two-sided matching markets is initiated by~\cite{liu2020competing}. They study the centralized setting, when the agents pass on their estimated preference to a single decision maker who then proposes to the arms. In this setting, the authors design a UCB based protocol that completely avoids collision (due to centralized proposals) and attains a $O(NK\log(T))$ regret with respect to the agent-pessimal stable matching. The papers of \cite{sankararaman2020dominate} and \cite{decentralized_jordan} are the closest to our work as both of them studies decentralized (the latter only partially) two sided matching markets. We have already mentioned their relation to this work, and how we improve upon them. Two sided markets under full information, commonly known as the \emph{stable marriage problem} was introduced in the seminal work of \cite{gale_shapely} where they established the notion of stability and provided the optimal algorithm to obtain a stable matching. Our results rely on the recent combinatorial characterization, namely $\alpha$-condition, of $\ucon$ in stable marriage problem by \cite{karpov2019necessary}.  Further, in the economics literature, uncertainty in two sided matching have been studied recently in \cite{johari_1, ashlagi_1} in directions tangential to our work. We provide a detailed related work in Appendix \ref{sec:appendix_related_work}.

\section{System Model}
\label{sec:system_model}
\textbf{Agents and arms.}  We have $N$ agents and $K \geq N$ arms. When agent $j\in [N]$ is matched to arm $k \in [K]$ it receives a reward sampled (independent of everything) from a latent distribution with support $[0,1]$, and  mean $\mu_{jk} \in [0,1]$. We assume that $\{ \mu_{jk} \}_{j, k}$ are all distinct. For each agent, the arm means impose a preference order over the arms, with higher means preferred over lower means. Similarly, every arm $k\in [K]$ has a total preference order over the arms $>_{\arm(k)}$; for $j,j' \in [N]$, if 
$\agent(j) >_{\arm(k)} \agent(j')$ then arm $k$ prefers agent $j$ over $j'$. When context is clear, we use $j$ for $\agent(j)$, and $k$ for $\arm(k)$. For any subset of agents and arms, the \emph{preference profile} is the preference order of the agents  restricted to the given set of arms, and vice versa. The game proceeds in $T$ rounds (value of $T$ unknown to the agents) where every agent simultaneously \emph{plays} one arm, and is either matched to that arm or is notified that it was not matched. In every round, each arm is matched to the most preferred agent playing that arm in that round (if any).

%

\textbf{Stable Matching.} Given the preference order for agents and arms, consider a matching denoted by the set $\mathbf{k^*, j^*}$, with $k^*_j$ denoting the matched arm for agent $j$, and $j^*_k$ denoting the matched agent for arm $k$. Under a \emph{stable-matching} there exist no two pairs $(j, k)$ and $(j', k')$ such that $k = k^*_j$ and $k'=k^*_{j'}$, but $\mu_{jk} < \mu_{jk'}$ and $\agent(j) >_{\arm(k')} \agent(j')$ -- if agent $j$ matches with $k'$ both improve their position.  An agent-optimal stable match is unique and is one where all agents obtain their respective best possible match among all possible stable matchings (see, \cite{gale_shapely}).


\textbf{Decentralized bandit with no information.} 
All agents have common information, that we have $N$ total agents and $K$ arms each labeled $1$ through $K$. In each round, every agent observes only the outcome of its  action and cannot observe the other agents' play/outcome. Thus, our feedback structure is \emph{fully decentralized} where in each round, an agent's decision to play an arm to play can depend only on common information before the start, its past actions and outcomes. Agents can however agree to a common protocol that map their observed outcomes to arms  in every round.

\textbf{Agent optimal regret.}
Total reward obtained by any agent is compared against that obtained by playing the agent optimal stable match in all rounds. Let $P_j(t) \in [K]$ be the arm played by agent $j$ in round $t$ and $M_j(t) \in \{0,1\}$ be the indicator random variable denoting whether agent $j$ was matched to $P_j(t)$. Let $k_j^{*}$ be the agent-optimal stable match of agent $j$. The $T$-round individual regret for an agent is
$$
    R_j(T) =  T\mu_{j k_j^{*}}  - \mathbb{E}\left[\sum_{t=1}^T M_j(t) \mu_{jP_j(t)}\right].
$$
The goal is to design a protocol that all agents follow to minimize their individual regret.

We make a few important remarks on the model.

\rem{1. Feedback structure.} Our feedback structure is same as~\cite{sankararaman2020dominate}, and more restrictive than the one proposed very recently in~\cite{decentralized_jordan} called decentralized bandit with partial information (see also, Section~\ref{sec:experiment}). In the latter, all the agents can observe the  matching between arms and agents in each round, alongside it's own reward or collision information. Further, the knowledge of preference order of the arms is explicitly assumed. 

\rem{2. Why agent optimal regret?}
 Our rationale for this is to compare against an oracle in which the arm-preferences of the agents are common information known to all agents, and each agent plays to maximize its individual total reward. This corresponds to a repeated game, in which, the set of pure strategy Nash equilibria corresponds to all agents playing a particular stable match in all rounds. The optimal equilibria that maximizes the reward for all agents simultaneously, corresponds to the agent-optimal stable matching. 

\rem{3. Model generalizations.} We focus on $N\leq K$, as for $N > K$ some agents will remain unmatched under full information, and have $0$ regret by definition. Our algorithms, can be modified easily to handle such cases. Further, we can admit any sub-Gaussian reward distribution with finite mean in place of rewards supported on $[0,1]$, with minor changes in our proof.  
\section{Achieving $O(\log^{(1+\varepsilon)}(T))$ Regret in Matching Markets for $\varepsilon>0$}
In this section, we give a simple Phased Explore then Commit (ETC) style algorithm that obtains an asymptotic per-agent regret of $O(\log^{1+\varepsilon}(T))$ for any given $\varepsilon > 0$. The algorithm proceeds in phases of exponentially growing lengths, with phase $i$ lasting $2^i$ rounds. In each phase $i \geq 0$, all agents explore for the first $i^{\varepsilon}$ rounds, and then subsequently exploit by converging to a stable matching through the Gale-Shapley strategy. The explorations are organized so that the agents avoid collision, which can be done in a simple decentralized fashion, where each agent gets assigned an unique id in the range $\{1,\cdots,N\}$ (Algorithm~\ref{algo:index_estimation} in Appendix~\ref{sec:app_algo}). The pseudo-code for phased ETC is given in Algorithm \ref{algo:etc}.


\begin{algorithm}[ht!]
\caption{Phased ETC Algorithm}
\label{algo:etc}
\begin{algorithmic}[1]
    \STATE  $\mathsf{Index} \gets$ {\ttfamily INDEX-ESTIMATION}()
	\FOR {$N+1 \leq t \leq T$ }
		\STATE $\phase \gets \lfloor \log_2(t-1) \rfloor$
		\IF {$t - 2^{\phase}+1 \leq K\lfloor {\phase}^{\varepsilon} \rfloor$} 
		    \STATE Play arm $(t+\mathsf{Index} - 2^{\phase}+1) \mod K$
		\ELSE 
		    \STATE Play {\ttfamily GALE-SHAPLEY-MATCHING}~\cite{gale_shapely} with arm-preference ordered by empirical means computed using all the explore samples thus far.
		\ENDIF
		\ENDFOR
\end{algorithmic}
\end{algorithm}



The minimum reward gap across all arms and agents is $\Delta = \min\{|\mu_{jk} - \mu_{jk'}|: j\in [N], k, k'\in [K]\}.$
\begin{theorem}\label{thm:etc}
If every agent runs Algorithm \ref{algo:etc} with input parameter $\varepsilon > 0$, then the regret of any agent $j$ after $T$ rounds satisfies
\begin{multline*}
    R_T^{(j)} \leq \frac{K(\log_2(T)+2)^{1+\varepsilon}}{1+\varepsilon} + (N^2+K)( \log_2(T) + 2) + 2^{ \left(\left(8/\Delta^2\right)^{1/\varepsilon} 4^{(1+\varepsilon)/\varepsilon} + 1 \right) } + \tfrac{e}{e-2}.
\end{multline*}
\end{theorem}
The proof is given in Appendix~\ref{sec:app_etc}. Several remarks are in order now on the regret upper bound of phased-ETC. 

\rem{1. Comparison with CA-UCB.}
From an asymptotic viewpoint (when $T \to \infty$), our result improves upon the recent result of \cite{decentralized_jordan}. Their proposed algorithm CA-UCB, achieves a regret of $O(\log^2(T))$ with respect to the \emph{agent-pessimal} stable matching. Theorem~\ref{thm:etc} shows that, even under our (restrictive) {\em decentralized bandit} model, the phased ETC algorithm achieves $O(\log^{1+\varepsilon}(T))$ regret with respect to the \emph{agent-optimal} stable matching.

\rem{2. Exponential dependence on gap.}
The constant in the regret bound has an exponential dependence on $\Delta^{-2}$, which limits it's applicability to practical values of $T$. We note that, the CA-UCB algorithm proposed in \cite{decentralized_jordan}, has an exponential dependence on $N^4$ that is multiplied with the $\log^2(T)$ term in their regret bound. It is an interesting open problem, to obtain an algorithm, that obtains $O(\log^{1+\varepsilon}(T))$ regret for the general matching bandit case, without any exponential dependence. 

In the following section, we present $\mainAlg$ - \name, a decentralized algorithm that achieves regret scaling as $O\left(\frac{\log(T)}{\Delta^2} \right)$ without any exponential dependence on problem parameters, whenever the underlying system satisfies \emph{uniqueness consistency}.

\section{Achieving $O(\log(T))$ Regret under Uniqueness Consistency}
We first introduce the uniqueness consistency in stable matching systems, and provide known combinatorial characterization of such systems, before presenting our algorithm. 


\noindent \textbf{Uniqueness Consistency.} The uniqueness consistency is an important subclass of graphs with unique stable-matching, and is defined as below after \cite{karpov2019necessary}.   
\begin{definition}[\cite{karpov2019necessary}]
A preference profile satisfies {\em uniqueness consistency} iff\\
(i) there exists a unique stable matching  $(\mathbf{k}^*,\mathbf{j}^*)$;\\
(ii) for any subset of arms or agents, the restriction of the preference profile on this subset with their stable-matched pair according to $(\mathbf{k}^*,\mathbf{j}^*)$ has a unique stable matching.
\end{definition}

The uniqueness consistency implies that even if an arbitrary subset of agent leave the system 
with their respective stable matched arms, we are left with a system with unique stable matching among the rest of agents and arms. This allows for any algorithm, which is able to identify at least one stable pair in a unique stable-matching system, to iteratively lead the system to the global unique stable matching. 

{\bf Combinatorial characterization.} 
In \cite{karpov2019necessary}, a necessary and sufficient condition for the uniqueness consistency of the stable matching system is established in the form of $\alpha$-condition (defined shortly). However, to connect $\alpha$-condition with the serial dictatorship studied in \cite{sankararaman2020dominate}, we present the sequential preference condition (SPC)~\cite{eeckhout2000uniqueness} which is a generalization of serial dictatorship, but a subclass of uniqueness consistency. The definition of SPC~\cite{eeckhout2000uniqueness} (which is stated in terms of men and women preferences) restated in our system where we have agents and arms is as follows.
\begin{definition}\label{def:spc}
{\em Sequential preference condition} (SPC) is satisfied iff  there is an  order of agents and arms so that 
\begin{align*}
    &\forall j \in [N], \forall k \in [K], k > j, : \mu_{jj} > \mu_{jk}, \text{ and } 
    &\forall k\leq N, \forall j \in [N], j > k: \agent\, j^*_k >_{\arm\, k} \agent\, j.
\end{align*}
\end{definition}

We next introduce a generalization of the SPC condition known as $\alpha$-condition, introduced in \cite{karpov2019necessary} recently, which is known to be the weakest sufficient condition for the uniqueness of stable matching.
Let $[K]_r =\{a_1,\dots, a_K\}$ and $[N]_r=\{A_1,\dots,A_N\}$ be permutations of the ordered sets $[K]$ and $[N]$, respectively, where the $k$-th arm in $[K]_r$ is the $a_k$-th arm in  $[K]$, and the $j$-th agent in $[N]_r$  is the $A_j$-th agent in $[N]$. We again restate the definition from \cite{karpov2019necessary} in our scenario. 
\begin{definition}\label{def:alpha}
The \textbf{$\boldsymbol{\alpha}$-condition} is satisfied iff there is a stable matching $(\mathbf{k}^*,\mathbf{j}^*)$, a left-order of agents and arms
$$\text{ s.t. }\forall j \in [N],\forall k > j, k \in [K]: \mu_{jk^*_j} > \mu_{jk},$$
and a (possibly different) right-order of agents and arms
\begin{align*}
&\text{ s.t. }\forall k < j \leq N, a_k \in [K]_r, A_j \in [N]_r: \agent\, A_{j^*_{a_k}}  >_{\arm\, a_k} \agent\, A_j.    
\end{align*}
\end{definition}
The following theorem in~\cite{karpov2019necessary} provides the characterization of uniqueness consistency, where unacceptable mates are absent (an unacceptable pair $(j,k)$ means agent $j$ does not accept arm $k$, or the vice versa).\footnote{\cite{karpov2019necessary} state the theorem for a system with same number of men and women. However, the theorem readily extends to our system where we have $K$ arms and $N$ agents with $N\leq K$. Indeed, the unmatched $(N-K)$ arms do not influence the stable matching under any arbitrary restriction of the preference profile.}

\begin{theorem}[\cite{karpov2019necessary}]\label{thm:unique}
If there is no unacceptable mates, then the $\alpha$-condition is a necessary and sufficient condition for the uniqueness consistency.
\end{theorem}
We end with a few remarks on these three systems.

\rem{1. Serial Dictatorship, SPC, and $\alpha$-condition.} 
As mentioned earlier, the $\alpha$-condition generalizes SPC. SPC is satisfied when the left order is identical to the right order in $\alpha$-condition. Further, SPC generalizes serial dictatorship, as the unique rank of the agents in the latter and their respective stable matched arms creates an SPC order. We present examples of the three systems.\\ 
1. This system is Serial dictatorship, SPC, and $\alpha$-condition
\begin{align*}
&\agent: a\hc 1>2>3, b\hc 1>2>3, c\hc 2>1>3,
&\arm: 1\hc a>b>c, 2\hc a>b>c, 3\hc a>b>c.
\end{align*}
2. This system is not Serial dictatorship as arms do not have a unique rank. But it satisfies SPC and $\alpha$-condition, with a valid SPC order $\{(a,1), (b,2), (c,3)\}$.
\begin{align*}
&\agent: a\hc 1>2>3, b\hc 1>2>3, c\hc 2>1>3,
&\arm: 1\hc a>b>c, 2\hc a>b>c, 3\hc a>c>b.
\end{align*}
3. The third system is not Serial dictatorship or SPC as there is no SPC order. But it satisfies $\alpha$-condition as a valid left order is $\{(a,1),(b,2),(c,3)\}$, and the corresponding right order is $\{(b,2),(c,3),(1,a)\}$.
\begin{align*}
&\agent: a\hc 1>2>3, b\hc 1>2>3, c\hc 2>1>3,
&\arm: 1\hc b>c>a, 2\hc b>a>c, 3\hc b>c>a.
\end{align*}
Currently, the $\alpha$-condition is the {\em weakest sufficient condition} for uniqueness of stable matching. Necessary condition for the uniqueness of stable matching remain elusive.

\rem{2. Stable matching under SPC and $\alpha$-condition.}
The  definition of SPC implies that the unique stable matching is obtained when, under the SPC order, for every $i \leq N$, the agent $i$ is matched to the arm $i$. Furthermore, under $\alpha$-condition for any $j\in [N]$, the agent $j$ is matched with arm $j$, and the agent $A_j$ is matched with the arm $a_j$ in the stable matching. We present a proof in the Appendix~\ref{sec:app_alpha}.

\subsection{\mainAlg: UCB Decentralized Dominate-Delete with Local Deletions}
In this section, we describe our main algorithm. The algorithm applies a novel technique called local deletions that interleaves with the phased global deletion strategy of UCB-D3 algorithm proposed in \cite{sankararaman2020dominate}. Algorithm~\ref{alg:UCBD4} describes this algorithm in detail.

\begin{algorithm}[ht!]
\caption{$\mainAlg$ algorithm (for an agent $j$)}
\label{alg:UCBD4}
\begin{algorithmic}[1]
    \STATE \textbf{Input:} Parameters $\beta \in (0, 1/K)$, and  $\gamma > 1$
    \STATE Set   $\mathsf{Index} \gets$ {\ttfamily INDEX-ESTIMATION}()
		\STATE  Global deletion $G_{j}[0] = \phi, \; \forall j \in [N]$
		\FOR {phase $\phase=1,2,\dots$}
	       \STATE Reset the collision counters $C_{jk}[i] = 0, \forall k \in[K]$
		   \STATE Delete dominated arms, $\mathcal{A}_j[\phase] \gets [K] \setminus G_{j}[\phase-1]$ 
	       \IF{$t < 2^\phase + N + NK(\phase-1)$} 
    		    \STATE Local deletion $L_{j}[\phase]\gets \{k: C_{jk}[i] \geq \lceil\beta2^i\rceil \}$ 
        		\STATE Play an arm 
        		 $P_j(t) \in \displaystyle \argmax_{k \in \mathcal{A}_j[\phase]\setminus L_{j}[\phase]} \left( \widehat{\mu}_{k, j}(t\mathtt{-}1) \mathtt{+} \sqrt{\tfrac{2\gamma \log(t)}{N_{k, j}(t-1)}} \right)$
        	    \IF{Arm $k=P_j(t)$ is matched}
        	        \STATE Update estimate $\widehat{\mu}_{k, j}$, and matching count $N_{k, j}$
        	    \ELSE
        	     \STATE Increase  collision counter $C_{jk}[i] \gets C_{jk}[i]+1$
        	    \ENDIF
		  \ELSIF{$t = 2^\phase + N + NK(\phase-1)$}
		       \STATE $\mathcal{O}_j[\phase]\mathtt{\gets}$ most matched arm so far in phase $\phase$ 
	           \STATE $G_{j}[\phase]\mathtt{\gets}$ {\ttfamily COMMUNICATION($\phase$, $\mathcal{O}[\phase]$)}
	     \ENDIF
		\ENDFOR
\end{algorithmic}
\end{algorithm}


At a high-level, the algorithm proceeds in phases as follows. At each phase $i$, every agent $j$ first updates its set of active arms for the phase given by $\mathcal{A}_j[\phase]$ by removing all the arms in the global dominated set $G_{j}[\phase-1]$. Then for the next $2^\phase$ time-steps, every agent plays the arm with the highest $\UCB$ from its active set $\mathcal{A}_j[\phase]$. Here, $\hat{\mu}_{jk}(t)$ denotes the estimated mean reward for, and $N_{jk}(t)$ the number of matches with,  arm $k$ by agent $j$ at the end of round $t$. Whenever the agent collides with an arm at least $\lceil\beta2^i\rceil $ times (tracked by collision counters $C_{jk}[\phase]$), it removes this arm from the active set. Finally, in the next $NK$ steps of this phase, each agent participates in a communication protocol (Algorithm~\ref{algo:communication} in Appendix~\ref{sec:app_algo}) to update the global dominated sets $G_{j}[\phase]$.

\section{Main Results}
We now present our main result in Theorem~\ref{thm:main} which show that $\mainAlg$ attains near optimal logarithmic regret when the stable matching satisfies uniqueness consistency.

{\bf System parameters.} We  introduce the following definitions first that will be used in the regret upper bound:\\
1. The {\em blocking agents} for agent $j \in [N]$ and arm $k \in [K]$,
$\mathcal{B}_{jk} :=  \{ j': \agent(j')>_{\arm(k)} \agent(j)\}.$
2. The  {\em dominated arms} for agent $j \in [N]$, $\mathcal{D}_j := \{k^*_{j'}:  j'\leq j-1\}.$
3. The {\em blocked non-dominated arms} for agent $j\in [N]$, $\mathcal{H}_j = \{k:  \exists j'\in \mathcal{B}_{jk}: k\notin \mathcal{D}_j \cup \mathcal{D}_{j'}\}.$
4. The {\em max blocking agent} for agent $j\in [N]$
$J_{\max}(j) = \max\left(j+1, \{j': \exists k\in \mathcal{H}_j, j'\in \mathcal{B}_{jk}\}\right).$
5. The {\em right-order mapping} for $\alpha$-condition for agent $j \in [N]$ is
$lr(j)$ so that $A_{lr(j)} = j$ with $A_{j}$ as defined in Definition~\ref{def:alpha}. We define the 
$lr_{\max}(j) = \max\{lr(j'): j'\leq j\}.$
6. The gaps of each agent $j$, is given as $\Delta_{jk} = (\mu_{jk^*_j} - \mu_{jk})$ which can be negative for some arms $k$. We define $\Delta_{\min} = \min\{\Delta_{jk}: j \in [N], k \in [K], \Delta_{jk}> 0\}$ as the minimum positive gap across arms and agents.

Some comments are due on some of the above definitions. The {\em dominated arms} is defined similar to \cite{sankararaman2020dominate}. These arms may have higher mean, hence higher long term UCB,  than the agent's stable matched arm. If not removed, the agent will incur linear regret due to collision from these arms which are played by blocking agents in steady state.  The {\em blocked non-dominated arms} are absent in serial dictatorship, but in the SPC and $\alpha$-condition  they emerge due to heterogeneity of  arm preference orders. These arms are not necessarily removed during the global deletion (i.e. through the set $G_j[i]$ in UCB-D4), and may create a deadlock for an agent where the agent keeps playing these arms without exploring them due to collisions.

{\bf Regret bounds.} We present the $O(\log(T))$ regret bound for the systems following uniqueness consistency in Theorem~\ref{thm:main}.  

\begin{theorem}\label{thm:main}
For a stable matching instance satisfying $\alpha$-condition (Definition~\ref{def:alpha}), suppose each agent follows $\mainAlg$ (Algorithm~\ref{alg:UCBD4}) with $\gamma > 1$ and $\beta \in (0,1/K)$, then the regret for an agent $j\in [N]$ is upper bounded by 
\begin{align*}
    \mathbb{E}[R_j(T)] &\leq \underbrace{\sum_{k\notin \mathcal{D}_j \cup k^*_j}\tfrac{8\gamma}{\Delta_{jk}}\left(\log(T)+\sqrt{\tfrac{\pi}{\gamma}\log(T)}\right)}_{\text{sub-optimal match}}
    + \underbrace{\sum_{k\notin \mathcal{D}_j} \sum_{j'\in \mathcal{B}_{jk}: k \notin \mathcal{D}_{j'}} \tfrac{8\gamma \mu_{k^*_j}}{\Delta^2_{j'k}}\left(\log(T)+\sqrt{\tfrac{\pi}{\gamma}\log(T)}\right)}_{\text{collision}} \\
    &+ \underbrace{(K - 1 + |\mathcal{B}_{jk^*_j}|)\log_2(T)}_{\text{communication}} + 
    + \underbrace{O\left(\tfrac{N^2K^2}{\Delta^2_{\min}} + (\beta|\mathcal{H}_j|f_\alpha(J_{\max}(j)) + f_{\alpha}(j)  - 2)2^{i^*}\right)}_{\text{transient phase, independent of $T$ }},\\
    \text{ where }
    &f_\alpha(j) = lr_{\max}(j) + j,\, i^*= \max\{8,i_1,i_2\},\\
    &i_1 = \min\{i: (N-1)\tfrac{10\gamma i}{\Delta^2_{min}} < \beta 2^{(i-1)}\}, \text{ and }
    i_2 = \min \{i: (N-1 + NK(i-1))\leq 2^{i+1}\}.
\end{align*}
Furthermore, if the system satisfies SPC (Definition~\ref{def:spc}) the $T$ independent term in the above regret bound can be replaced by $O\left(\tfrac{N^2K^2}{\Delta^2_{\min}} + ( \beta|\mathcal{H}_j| J_{\max}(j) + j-1)2^{i^*} \right)$.
\end{theorem}
We conclude with some remarks about our main result.

\rem{Scaling of regret bounds.} Our regret bounds in Theorem~\ref{thm:main} are desirable for the following reasons.\\
1. {\em Dependence on the arm gaps and horizon.} Our regret upper bound is near optimal in the arm gaps and horizon, as it matches with the regret lower bound of $\Omega\left(\tfrac{\log(T)}{\Delta^2_{\min}}\right)$ in Corollary 6 of \cite{sankararaman2020dominate}.\\
2. {\em Polynomial dependence.} The constant term associated with the regret bound   is polynomial in all the system parameters, including the minimum gap, as 
$$2^{i^*}= O\left(\max\left\{\tfrac{N}{\Delta^2_{min}}\log(\tfrac{N}{\Delta^2_{min}}), NK \log(NK)\right\}\right).$$ 
Hence, UCB-D4 regret bounds, under uniqueness consistency, has no exponential dendence on system parameters.\\
3. {\em Dependence on preference profile.} The preference profile influences the regret mainly in three ways. First, each agent deletes the dominated arms $\mathcal{D}_j$ so incurs no long term regret for those arms. The first term in the regret bound scales as $(K-j)$. Second, for any agent and any non-dominated arms (which are not globally deleted) the  blocking agents during exploration creates collision leading to the second term scaling as at most $|\mathcal{H}_j \setminus \mathcal{D}_j|$. Finally, the constant in the regret bound scales linearly with the order of the agent and all of its blocking agents ($j + J_{\max}(j)$) in SPC. For the uniqueness consistency, the constant depends on both the left and right order ($f_{\alpha}(j) = j+lr_{\max}(j)$) of the agent and all of its blocking agents $\left(f_{\alpha}(j) + f_{\alpha}(J_{\max}(j))  \right)$.

\rem{Local deletion threshold.} The local deletion threshold is set as $\beta\times \text{phase length}+\Theta(1)$ with $\beta < 1/K$. Increasing the threshold leads to higher regret until local deletion vanishes. This happens as more collision is allowed until an arm is deleted. But higher threshold allows for quick detection of the stable matched arm. However, decreasing the threshold leads to a more aggressive deletion leading to lower regret from collision per phase, at a cost of longer detection time for the stable matched arm. 
In particular, if instead we set local deletion threshold as  $\Theta(\mathrm{polylog}(\text{phase length}))$ then the constant $2^{i^*}$ becomes exponential, i.e. $\exp(N/\Delta^2_{\min})$.  At one extreme, for $\beta \geq 1/K$, the global deletion may stop freezing as the stable matched arm for an agent is not guaranteed to emerge as the most matched arm. In the other extreme, for a threshold of $\Theta(\log(\text{phase length}))$ with large enough probability the stable matched arm may get locally deleted.  In both extremes, the $\mainAlg$ ceases to work.

\subsection{Key Insights into the Proof of Theorem~\ref{thm:main}}
The full proof of Theorem~\ref{thm:main} is presented in the Appendix~\ref{sec:app_alpha}.  We first present why UCB-D3~\cite{sankararaman2020dominate} that  works for the serial dictatorship setting fails under SPC. 

{\bf Deadlocks beyond serial dictatorship.}  Under SPC while running UCB-D3~\cite{sankararaman2020dominate} the {\em blocked non-dominated arms} cause deadlock where two agents are unable to sample their respective best stable matched arm leading to linear regret. This is best explained by an example. Let us consider the $3$ agent ($a$, $b$, and $c$), and $3$ arm ($1$, $2$, and $3$) system with the preference lists given as 
$$\begin{aligned}
&\agent: a\hc 1>2>3, b\hc 2>1>3, c\hc 3>1>2,
&\arm: 1\hc a>b>c, 2\hc a>b>c, 3\hc a>c>b.
\end{aligned}$$
This system satisfies SPC with the agent optimal stable matching $\{(a,1), (b,2),(c,3)\}$,  but it is not a serial dictatorship. If we run the UCB-D3 algorithm in~\cite{sankararaman2020dominate} then the agent $a$ matches with arm $1$,and agent  $b$  and $c$ both delete arm  $1$ using global deletion. However, without additional coordination, with non negligible probability, agent $b$ may not match with arm $3$ (collision with agent $c$), and agent $c$ may not match with arm $2$ (collision with agent $b$). No further global deletion is guaranteed resulting in a deadlock, hence  linear regret (see, Section~\ref{sec:experiment}). 

We first focus on the proof for SPC condition before tackling $\alpha$-condition, as the latter builds on the former. 
\begin{proof}[\rem{Proof Sketch for SPC}] There are two main components of the proof,  {\em inductive freezing of stable matching pairs}, and {\em vanishing of local deletion}, both in expected constant time. 


{\em Inductive freezing of stable matching pairs.} Local deletion of arms with more than $(\beta \times \text{phase length})$ collisions, for $\beta < 1/K$, ensures the agent ranked $1$ (under SPC) removes all the {\em blocked  non-dominated arms}, thus it never gets stuck in a deadlock. This gives agent $1$ opportunity to detect arm $1$ as its most matched arm in expected constant time, as in the presence of arm $1$, agent $1$ matches with any non-deleted sub-optimal arm  $O(\log(\text{phase length}))$ times with high probability under UCB dynamics. Note that agent $1$ never deletes arm $1$ (either globally or locally)  due to the SPC condition. Agent $2$ next deletes dominated arm $1$, and for similar reasons as above matches most with arm $2$. The proof is completed using an induction over the agents following the SPC order, where in agent $j$ freezing happens in $O(ji^*)$ phases. 

{\em Constant time vanishing of local deletion.}
We establish that the local deletion vanishes in expected constant time. Indeed, when all the agents settle to their respective stable matched arms,  the sub-optimal play is limited to $O(log(\text{phase length}))$, hence total collision to blocked non-dominated arm is also logarithmic in phase length. This leads to the vanishing of local deletion, which requires $(\beta\times\text{phase length})$ collisions,  with an extra $O( \beta|\mathcal{H}_j| J_{\max}(j) 2^{i^*})$ regret, where $J_{\max}(j)$ is the maximum blocking agent. 

{\em Stable regime.} The terms that grow with $T$ as $O(\log(T))$ are accounted for by keeping count of expected (1)  number of sub-optimal matches, (2) number of collisions from blocking agents, and (3) the regret due to communication sub-phases. 
\end{proof}

\noindent \textbf{Uniqueness consistency - A tale of two orders.} When moving from  SPC instances to uniqueness consistency (equivalently $\alpha$-condition due to Theorem~\ref{thm:unique}) we no longer have the simple inductive structure that we leverage in the proof sketch of SPC. Under $\alpha$-condition we have two orders instead. The left order states when arm $1$ to arm $(j-1)$ are removed agent $j$ has arm $j$ as the most preferred arm. Whereas, the right order states when agent $A_1$ to agent $A_{(k-1)}$ (recall $\{A_1, \dots, A_N\}$ is a separate permutation of the agents) are removed arm $a_k$ prefers agent $A_k$. Global deletion here must follow left order. However, unlike SPC the set of blocking agents for agent $1$ and arm $1$ (i.e. $\mathcal{B}_{11}$) is nonempty. Thus, agent $1$ cannot get matched with arm $1$ majority of time unless the agents in $\mathcal{B}_{11}$ stop playing arm $1$. We next show how this is resolved.

\begin{proof}[\rem{Proof Sketch of $\alpha$-condition}]
We show due to local deletion an {\em inductive warmup of stable matching pairs} precedes the two phases mentioned in the proof sketch of SPC. 

{\em Inductive warmup of stable matching pairs.} We begin with an observation that holds for every stable matching: for any two stable matching pair $(j,k)$, $(j',k')$, the arm $k$ is either sub-optimal (has mean reward lesser than $k'$) for agent $j'$, or arm $k$ prefers agent $j$ over agent $j'$. Hence, in our system,  for any $2\leq j \leq N$, either (a) arm $a_j$ (stable matched arm for agent $A_j$) is suboptimal for agent $A_1$, or (b) arm $a_j$ prefers $A_j$ over $A_1$. In case (b) arm $A_1$ never causes collision in $a_j$ for agent $A_j$.
Now suppose case (a) holds. Due to $\alpha$-condition, the agent $A_1$ is the most preferred for arm $a_1$.  Therefore, arm $a_1$ is always available to agent $A_1$, and  $a_1$ has higher mean than $a_j$. This implies the UCB algorithm plays arm $a_j$ only $O(\log(\text{phase length}))$ times with high probability in any phase. Once this happens, arm $a_2$ is almost always available to $A_2$, and the induction sets in. This is used to prove $A_j$ is warmed up in $O(ji^*)$ phases, or agent $j$ is warmed up in $O(lr(j)i^*)$ phases.  Once agent $1$ is warmed up the inductive freezing starts as it matches with arm $1$. The rest closely follows the proof of SPC. 
\end{proof}

\section{Numerical Simulations}
\label{sec:experiment}
In this section, we present our numerical simulations with $5$ agents and $6$ arms. Additional results with larger instances are deferred to the Appendix~\ref{sec:app_experiments}.  

\textbf{Baselines.} We use two baselines with their own feedback.\\
{1. \em Centralized UCB (UCB-C)} is proposed in \cite{liu2020competing} for the centralized feedback setting, where each agent in every round submits it's preference order (based on UCB indices) to a centralized agent who assigns the agent optimal matching under this preference, and the rewards are observed locally.\\
{2. \em Collision Avoidance UCB (CA-UCB)} is proposed in \cite{decentralized_jordan} for the {\em decentralized with partial information} setting where at each round all agents observe the player matched to each arm, but rewards are observed locally. 

The centralized setting is the most relaxed where no collision happens, followed by the partial decentralized setting where each round the matching of arms can be learned without any collision. Both are relaxed compared to our {\em decentralized with no information} setting where any global information can be obtained only through collisions.

\textbf{Results.}
We generate random instances to compare the performance of UCB-C, CA-UCB and UCB-D4 in their respective settings. Since \cite{decentralized_jordan} does not mention how to set their hyper-parameter $\lambda$, we report the best result by running a grid search over $\lambda$. For UCB-D4, we use $\beta = 1/2K$ and $\gamma = 2$. We simulate all the algorithms on the same sample paths, for a total 50 sample paths and report mean, 75\% and 25\% agent-optimal regret.

Figure~\ref{fig:etcMain} shows that in a general instance phased ETC outperforms CA-UCB even with a restricted feedback, whereas, as expected, UCB-C outperforms phased ETC.  
Figure~\ref{fig:alphaMain} shows that when uniqueness condition holds UCB-D4, despite the restricted feedback, outperforms CA-UCB,  while it is comparable to the centralized UCB-C.

\begin{figure}[h]
\centering
    \begin{subfigure}[b]{0.5\textwidth}
    \includegraphics[width=0.9\linewidth]{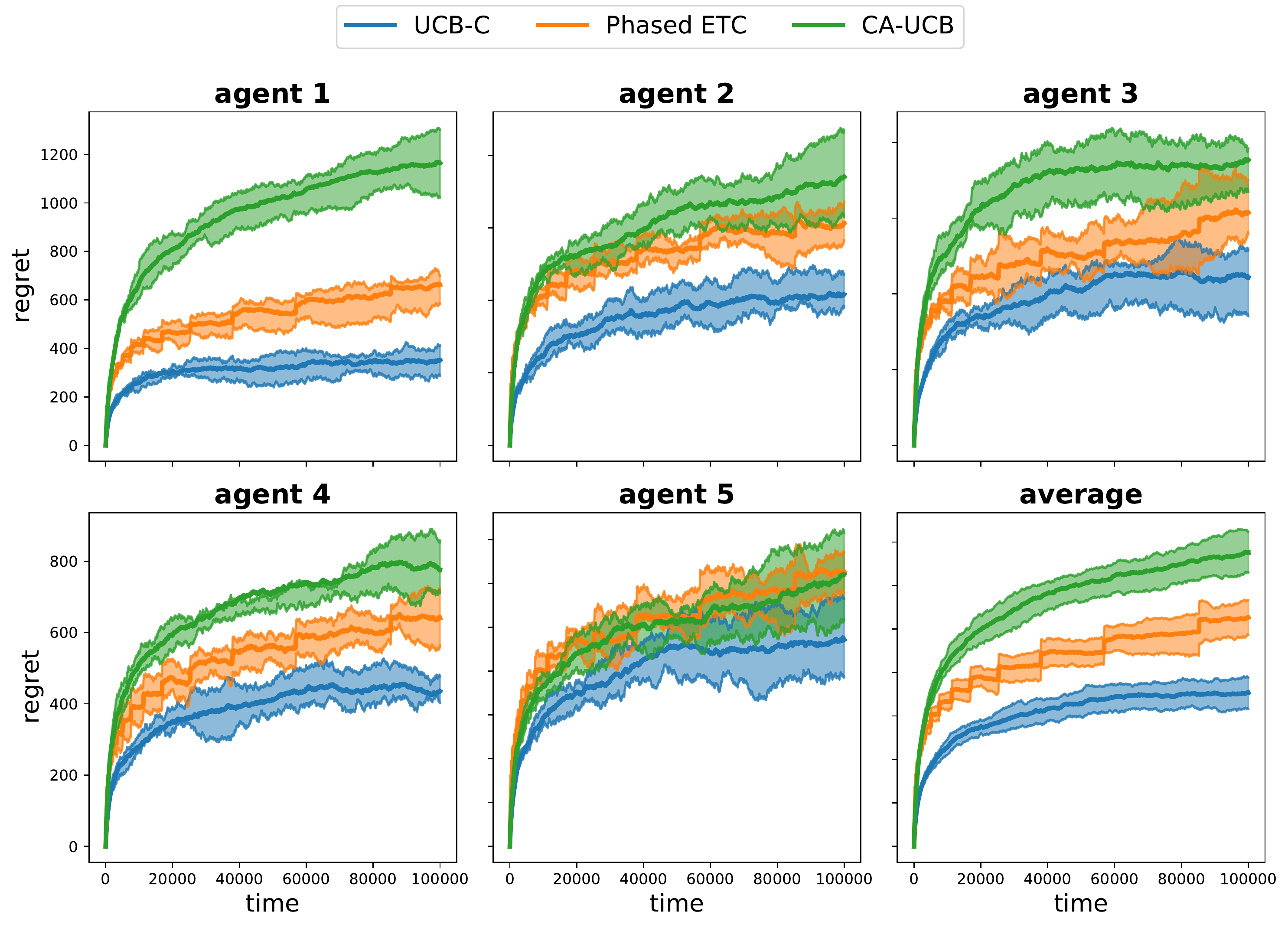}
    \caption{A general instance with $N=5$, and $K=6$. }
    \label{fig:etcMain}
    \end{subfigure}%
    ~
    \begin{subfigure}[b]{0.5\textwidth}
    \includegraphics[width=0.9\linewidth]{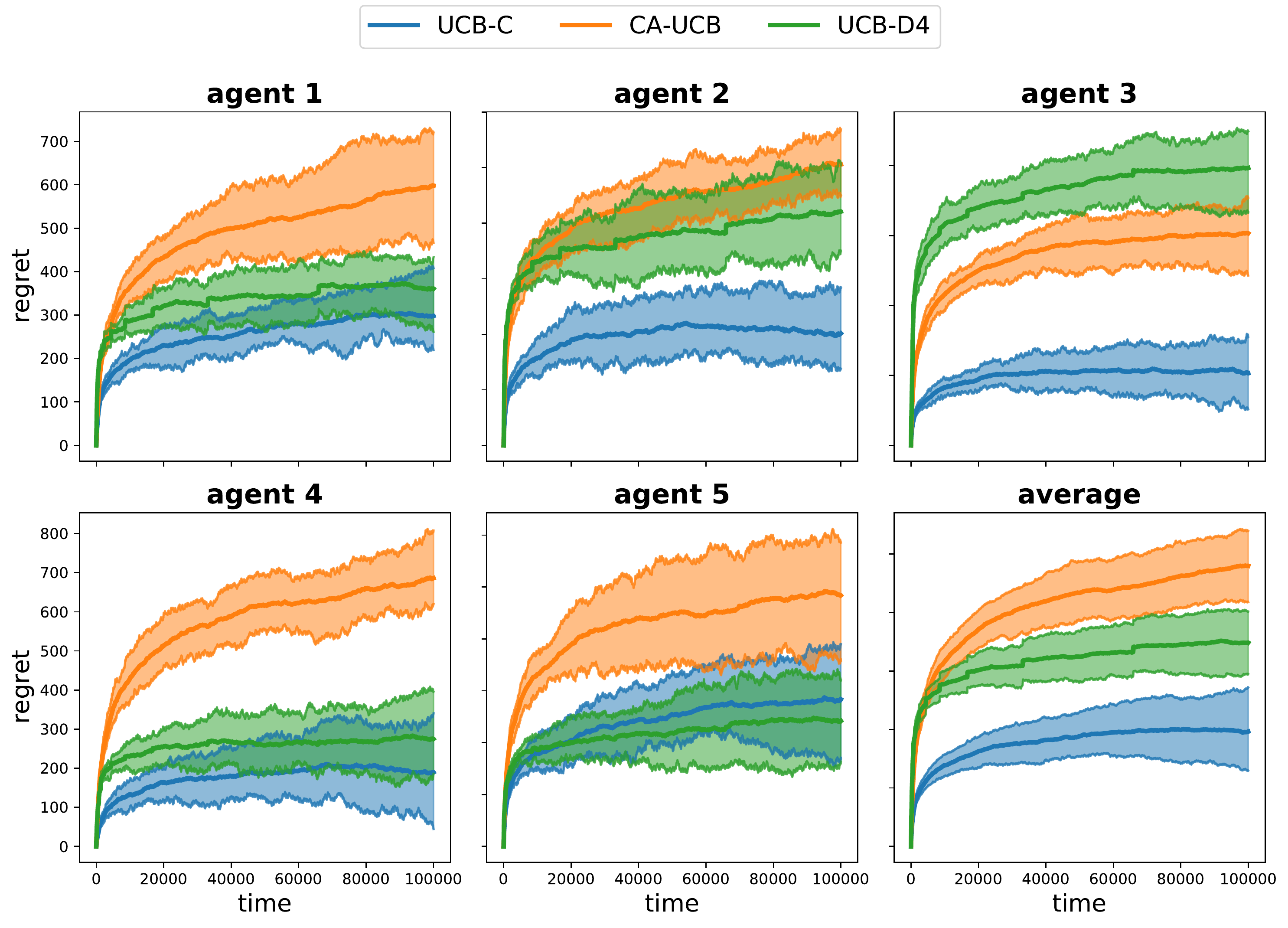}
    \caption{An $\alpha$-condition instance with $N=5$, and  $K=6$.}
    \label{fig:alphaMain}
    \end{subfigure}
    \caption{Performance comparison of algorithms}
\end{figure}
\vspace{-1em}
\begin{figure}[H]
\centering
  \includegraphics[width=0.45\linewidth]{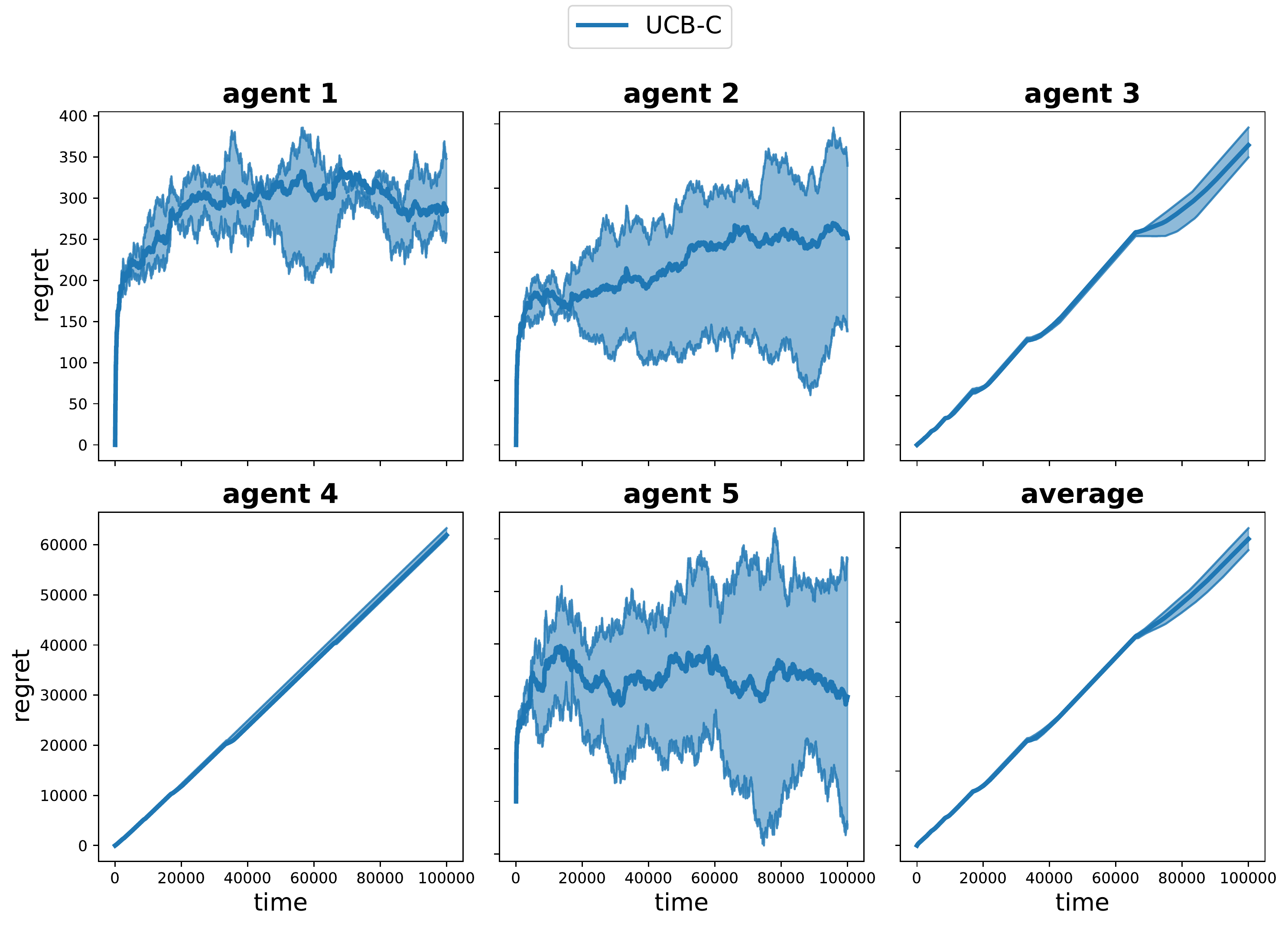}
\caption{Linear regret of UCB-D3 for an SPC instance.}
\label{fig:ucbd3Main}
\end{figure}
\vspace{-1em}
{\bf Linear Regret of UCB-D3 in SPC.} Figure~\ref{fig:ucbd3Main} shows that even when the SPC condition holds, UCB-D3 may result in linear regret, emphasizing the importance of UCB-D4.

\section{Conclusion}
In this paper, we gave a simple phase-based ETC algorithm, that achieves regret $O(\log^{1+\varepsilon}(T))$ with respect to the agent optimal stable matching, for any $\varepsilon > 0$, and any market. This improves on prior work in two ways \textemdash first the ETC compares against the agent-optimal stable matching, while prior works could only compare against the pessimal stable matching. Second, we show an algorithm achieving $O(\log^{1+\varepsilon}(T))$, an improvement over $O(\log^2(T))$ regret established prior. Further, we provided a new algorithm, UCB-D4, that achieves $O(\log(T))$ regret for the class of uniqueness-consistent markets. This class is the largest known class today to admit a unique stable matching. This result extends the previously known result, where $O(\log(T))$ regret was only known to be achieved in the restrictive special case of serial dictatorship. The paper however leaves open the intriguing question of whether there exists an algorithm achieving $O(\log(T))$ regret for genral matching markets.

\bibliography{ref}

\begin{thebibliography}{10}

\bibitem{aridor2020competing}
Guy Aridor, Yishay Mansour, Aleksandrs Slivkins, and Zhiwei~Steven Wu.
\newblock Competing bandits: The perils of exploration under competition.
\newblock {\em arXiv preprint arXiv:2007.10144}, 2020.

\bibitem{ashlagi_1}
Itai Ashlagi, Anilesh~K Krishnaswamy, Rahul Makhijani, Daniela Saban, and
  Kirankumar Shiragur.
\newblock Assortment planning for two-sided sequential matching markets.
\newblock {\em arXiv preprint arXiv:1907.04485}, 2019.

\bibitem{wireless_mab}
Orly Avner and Shie Mannor.
\newblock Multi-user lax communications: a multi-armed bandit approach.
\newblock In {\em IEEE INFOCOM 2016-The 35th Annual IEEE International
  Conference on Computer Communications}, pages 1--9. IEEE, 2016.

\bibitem{got_bandits}
Ilai Bistritz and Amir Leshem.
\newblock Game of thrones: Fully distributed learning for multiplayer bandits.
\newblock {\em Mathematics of Operations Research}, 2020.

\bibitem{sic_mmab}
Etienne Boursier and Vianney Perchet.
\newblock Sic-mmab: synchronisation involves communication in multiplayer
  multi-armed bandits.
\newblock {\em arXiv preprint arXiv:1809.08151}, 2018.

\bibitem{bandit_nash}
Etienne Boursier and Vianney Perchet.
\newblock Selfish robustness and equilibria in multi-player bandits.
\newblock In {\em Conference on Learning Theory}, pages 530--581. PMLR, 2020.

\bibitem{bandit_nash_2}
Simina Br{\^a}nzei and Yuval Peres.
\newblock Multiplayer bandit learning, from competition to cooperation.
\newblock {\em arXiv preprint arXiv:1908.01135}, 2019.

\bibitem{bubeck2012regret}
S{\'e}bastien Bubeck and Nicolo Cesa-Bianchi.
\newblock Regret analysis of stochastic and nonstochastic multi-armed bandit
  problems.
\newblock {\em Foundations \& Trends in Machine Learning}, 2012.

\bibitem{info_share}
Swapna Buccapatnam, Jian Tan, and Li~Zhang.
\newblock Information sharing in distributed stochastic bandits.
\newblock In {\em 2015 IEEE Conference on Computer Communications (INFOCOM)},
  pages 2605--2613. IEEE, 2015.

\bibitem{gosine}
Ronshee Chawla, Abishek Sankararaman, Ayalvadi Ganesh, and Sanjay Shakkottai.
\newblock The gossiping insert-eliminate algorithm for multi-agent bandits.
\newblock In {\em International Conference on Artificial Intelligence and
  Statistics}, pages 3471--3481. PMLR, 2020.

\bibitem{wireless_mab_2}
Sumit~J Darak and Manjesh~K Hanawal.
\newblock Multi-player multi-armed bandits for stable allocation in
  heterogeneous ad-hoc networks.
\newblock {\em IEEE Journal on Selected Areas in Communications},
  37(10):2350--2363, 2019.

\bibitem{collab_6}
Abhimanyu Dubey et~al.
\newblock Cooperative multi-agent bandits with heavy tails.
\newblock In {\em International Conference on Machine Learning}, pages
  2730--2739. PMLR, 2020.

\bibitem{eeckhout2000uniqueness}
Jan Eeckhout.
\newblock On the uniqueness of stable marriage matchings.
\newblock {\em Economics Letters}, 69(1):1--8, 2000.

\bibitem{gale_shapely}
David Gale and Lloyd~S Shapley.
\newblock College admissions and the stability of marriage.
\newblock {\em The American Mathematical Monthly}, 69(1):9--15, 1962.

\bibitem{mab_ads}
Eshcar Hillel, Zohar Karnin, Tomer Koren, Ronny Lempel, and Oren Somekh.
\newblock Distributed exploration in multi-armed bandits.
\newblock {\em arXiv preprint arXiv:1311.0800}, 2013.

\bibitem{johari_1}
Ramesh Johari, Vijay Kamble, and Yash Kanoria.
\newblock Matching while learning.
\newblock {\em Operations Research}, 2021.

\bibitem{kalathil}
Dileep Kalathil, Naumaan Nayyar, and Rahul Jain.
\newblock Decentralized learning for multiplayer multiarmed bandits.
\newblock {\em IEEE Transactions on Information Theory}, 60(4):2331--2345,
  2014.

\bibitem{karpov2019necessary}
Alexander Karpov.
\newblock A necessary and sufficient condition for uniqueness consistency in
  the stable marriage matching problem.
\newblock {\em Economics Letters}, 178:63--65, 2019.

\bibitem{colab_social_nets}
Ravi~Kumar Kolla, Krishna Jagannathan, and Aditya Gopalan.
\newblock Collaborative learning of stochastic bandits over a social network.
\newblock {\em IEEE/ACM Transactions on Networking}, 26(4):1782--1795, 2018.

\bibitem{collab_5}
Peter Landgren, Vaibhav Srivastava, and Naomi~Ehrich Leonard.
\newblock Distributed cooperative decision making in multi-agent multi-armed
  bandits.
\newblock {\em Automatica}, 125:109445, 2021.

\bibitem{bandits_resource_alloc}
Maialen Larrnaaga, Urtzi Ayesta, and Ina~Maria Verloop.
\newblock Dynamic control of birth-and-death restless bandits: Application to
  resource-allocation problems.
\newblock {\em IEEE/ACM Transactions on Networking}, 24(6):3812--3825, 2016.

\bibitem{lattimore2020bandit}
Tor Lattimore and Csaba Szepesv{\'a}ri.
\newblock {\em Bandit algorithms}.
\newblock Cambridge University Press, 2020.

\bibitem{liu2020competing}
Lydia~T Liu, Horia Mania, and Michael Jordan.
\newblock Competing bandits in matching markets.
\newblock In {\em International Conference on Artificial Intelligence and
  Statistics}, pages 1618--1628. PMLR, 2020.

\bibitem{decentralized_jordan}
Lydia~T Liu, Feng Ruan, Horia Mania, and Michael~I Jordan.
\newblock Bandit learning in decentralized matching markets.
\newblock {\em arXiv preprint arXiv:2012.07348}, 2020.

\bibitem{practical_mab_algo}
Abbas Mehrabian, Etienne Boursier, Emilie Kaufmann, and Vianney Perchet.
\newblock A practical algorithm for multiplayer bandits when arm means vary
  among players.
\newblock In {\em International Conference on Artificial Intelligence and
  Statistics}, pages 1211--1221. PMLR, 2020.

\bibitem{musical_chairs}
Jonathan Rosenski, Ohad Shamir, and Liran Szlak.
\newblock Multi-player bandits--a musical chairs approach.
\newblock In {\em International Conference on Machine Learning}, pages
  155--163. PMLR, 2016.

\bibitem{sankararaman2020dominate}
Abishek Sankararaman, Soumya Basu, and Karthik~Abinav Sankararaman.
\newblock Dominate or delete: Decentralized competing bandits with uniform
  valuation.
\newblock {\em arXiv preprint arXiv:2006.15166}, 2020.

\bibitem{social_learning}
Abishek Sankararaman, Ayalvadi Ganesh, and Sanjay Shakkottai.
\newblock Social learning in multi agent multi armed bandits.
\newblock {\em Proceedings of the ACM on Measurement and Analysis of Computing
  Systems}, 3(3):1--35, 2019.

\bibitem{slivkins2019introduction}
Aleksandrs Slivkins.
\newblock Introduction to multi-armed bandits.
\newblock {\em Foundations and Trends{\textregistered} in Machine Learning},
  12(1-2):1--286, 2019.

\end{thebibliography}
\bibliographystyle{plain}

\clearpage
\begin{appendices}
\onecolumn
\allowdisplaybreaks
\section{Sub-Routines used in the Algorithms}\label{sec:app_algo}
\label{appendix_algo_support}

\begin{algorithm}[H]
\caption{{\ttfamily INDEX-ESTIMATION}()}
\label{algo:index_estimation}
\begin{algorithmic}[1]
    \STATE  {{\ttfamily Index}} $\gets$ {\ttfamily N}
    \STATE {\ttfamily Arm} $\gets 1$
	\FOR {$1 \leq t \leq N-1$ }
		\STATE Play Arm labeled {\ttfamily Arm}
		\IF {Matched \textbf{ AND }  {\ttfamily Arm} $== 1$ }
		    \STATE {{\ttfamily Index}} $\gets t$
		    \STATE {\ttfamily Arm} $\gets 2$
		    \ENDIF
		\ENDFOR
	\STATE \textbf{Return} {\ttfamily Index}
\end{algorithmic}
\end{algorithm}

In Algorithm \ref{algo:index_estimation}, we give a simple algorithm by which every agent in a decentralized fashion, can estimate an unique rank. As the arms are labeled, the agents agree to the protocol, that at the beginning of the game, they will play arm labeled $1$, until it gets matched for the first time. The index of the agent is the time at which it matches with arm $1$. Subsequently, the agent will play arm $2$ in the remaining time. Thus, the id estimated by an agent is its relative rank at arm $1$, i.e. $\mathsf{Index}(j) = rank(j,1)$ which is unique among the agents as for any arm there is no tie in its preference.


\begin{algorithm}[H]
\caption{{\ttfamily COMMUNICATION}()}
\label{algo:communication}
\begin{algorithmic}[1]
    \STATE \textbf{Input:} Phase number $\agent j, \phase \in \{1, 2, \ldots, \}$, max played arms $\mathcal{O}[\phase]$
    \FOR{$t=1$ to $NK-1$}
        \IF{$K(\mathsf{Index}(j)-1) \leq t \leq K\mathsf{Index}(j)-1$}
            \STATE Play arm $P_j(t) = (t \mod K ) +1$ 
			\IF{ Collision Occurs}
				\STATE $\mathcal{C} \gets \mathcal{C} \cup \{ I_j(t)\}$
			\ENDIF
        \ELSE
			\STATE Play arm $P_j(t) = \mathcal{O}_j[\phase]$
		\ENDIF
    \ENDFOR
    \STATE \textbf{Return} $\mathcal{C}$
\end{algorithmic}
\end{algorithm}

Algorithm~\ref{algo:communication} allows each agent $j$ to learn the dominated arms -- the arms which are most played by at least one agent $j'\neq j$ such that $j'>_{\mathcal{O}_{j'}[\phase]} j$. We argue that if there exists such an agent $j'$ then
$\mathcal{C} \ni \mathcal{O}_{j'}[\phase]$. Assume an arbitrary such agent $j'$. At the time $t = (K(\mathsf{Index}(j)-1) + \mathcal{O}_{j'}[\phase]-1)$ the agent $j'$ and $j$ plays arm $\mathcal{O}_{j'}[\phase]$, but agent $j$ collides as $j'>_{\mathcal{O}_{j'}[\phase]} j$. This ensures $\mathcal{C} \subseteq \{\mathcal{O}_{j'}[\phase]: j'>_{\mathcal{O}_{j'}[\phase]} j\}$. Also, if there is no $j'$ such that  $j'>_{\mathcal{O}_{j'}[\phase]} j$, then there is no collision during $K(\mathsf{Index}(j)-1) \leq t \leq K\mathsf{Index}(j)-1$ when $\mathcal{C}$ is updated (there can be collision in other times). So we have  $\mathcal{C} = \{\mathcal{O}_{j'}[\phase]: j'>_{\mathcal{O}_{j'}[\phase]} j\}$ which proves the correctness for the  Algorithm~\ref{algo:communication}.
\section{Additional Related Work on Multi-Agent Bandits}
\label{sec:appendix_related_work}

A popular line of work in competitive multi-agent bandits (as in this paper) is the \emph{colliding bandits model}, where if two or more agents play the same arm in the same round, then \emph{all} of them get blocked. Such models study spectrum sharing in wireless networks \cite{kalathil, musical_chairs, got_bandits}, and recently sophisticated centralized and de-centralized algorithms were  developed to obtain the right $O(\log(T))$ regret \cite{sic_mmab, practical_mab_algo}. Our model fundamentally differs from the colliding bandits, because if multiple agents play the same arm simultaneously, one of them receives a reward while the others do not. Thus, the developments therin \cite{sic_mmab, practical_mab_algo} are inapplicable to our problem. 
From an application side, bandits have been used for resource allocation in networks \cite{wireless_mab_2, bandits_resource_alloc, wireless_mab}. This line of work does not fall under our matching bandit model and are mostly centralized systems.
Another related line of work is that of collaborative multi-agent bandits, where agents are not competing for resources, but aim to maximize group reward by minimal communications \cite{colab_social_nets, gosine, social_learning, info_share, collab_5, collab_6}.
\section{Proof of Theorem \ref{thm:etc}}\label{sec:app_etc}
\label{appendix_proof_etc}

\begin{proof}[Proof of Theorem \ref{thm:etc}]
We use the Hoeffding's bound and linearity of expectation to establish this result. We collect some useful observations from the description of the algorithm.
\begin{itemize}[noitemsep]
\item The total number of phases until time $T$ is at-most $ \lceil \log_2(T) \rceil + 1 $. 
    \item The total number of explore samples by any agent after the explore part of phase $i$ is at-least $\frac{(i-1)^{1+\varepsilon}}{1+\varepsilon} - i$.
    \item Each agent experiences no more than $N^2$ collisions in a phase.
\end{itemize}
The first point follows from the fact that phase $i$ lasts for $2^i$ rounds. The second point follows from the fact that in any phase $i$. the first $i^{\varepsilon}$ rounds are used for exploration. The last point follows from classical results on Gale-Shapley matching, which takes at-most $N^2$ rounds.

In any phase $i$, denote by the event $\mathcal{E}_i$ to be the one in which, every agent correctly estimated the ordering of its arms, at the end of the explore portion of phase $i$. The following two propositions help us prove the regret bound
\begin{proposition}
\begin{equation*}
    R_T^{(j)} \leq \sum_{i=0}^{\lceil \log_2(T) \rceil +1} 2^i\mathbb{P}(\mathcal{E}_i^{\complement}) +  K\frac{(\log_2(T)+2)^{1+\varepsilon}}{1+\varepsilon} + (N^2+K)( \log_2(T) + 2).
\end{equation*}
\label{prop:etc_linearity_expectation}
\end{proposition}

\begin{proposition}
\begin{equation}
    \mathbb{P}[\mathcal{E}_i^{\complement}] \leq 2N K \exp \left( - \left( \frac{(i-1)^{1+\varepsilon}}{1+\varepsilon} - i\right) \frac{\Delta^2}{2} \right).
\end{equation}
\label{prop:etc_chernoff}
\end{proposition}

Before giving the proofs of these propositions, we show how these aid in bounding the regret. 
Denote by $i_{\Delta}$ as 
\begin{align*}
    i_{\Delta} := \min \left\{ u : \forall i \geq u, \left( \frac{(i-1)^{1+\varepsilon}}{1+\varepsilon} - i\right) \frac{\Delta^2}{4} \geq i \right\}.
\end{align*}
It is easy to verify that 
\begin{align}
    i_{\Delta} \leq \left( \frac{8}{\Delta^2}\right)^{\frac{1}{\varepsilon}} 4^{\frac{1+\varepsilon}{\varepsilon}}.
    \label{eqn:etc_proof_i_delta}
\end{align}
Now, from Propositions \ref{prop:etc_linearity_expectation} and \ref{prop:etc_chernoff}, we get the following
\begin{equation}
    R_T^{(j)} \leq 2NK \sum_{i=0}^{\lceil \log_2(T) \rceil +1} 2^i e^{  -2 \left( \frac{(i-1)^{1+\varepsilon}}{1+\varepsilon} - i\right) \frac{\Delta^2}{4}   } +  K\frac{(\log_2(T)+2)^{1+\varepsilon}}{1+\varepsilon} + (N^2+K)( \log_2(T) + 2).
    \label{eqn:etc_regret_final}
\end{equation}
From the definition of $i_{\Delta}$, we have
\begin{align}
\sum_{i=0}^{\lceil \log_2(T) \rceil +1} 2^i e^{ -2 \left( \frac{(i-1)^{1+\varepsilon}}{1+\varepsilon} - i\right) \frac{\Delta^2}{4}  } &\leq \sum_{i=0}^{i_{\Delta}}2^i + \sum_{i \geq 0}\left( \frac{2}{e} \right)^i, \nonumber \\
&\leq 2^{i_{\Delta}+1} + \frac{e}{e-2}.
\label{eqn:etc_proof_regret_pre_final}
\end{align}
Substituting Equation (\ref{eqn:etc_proof_regret_pre_final}) into Equation (\ref{eqn:etc_proof_pre_final}), and bounding $i_{\Delta}$ with Equation (\ref{eqn:etc_proof_i_delta}), yields the result.
\end{proof}

\begin{proof}[Proof of Proposition \ref{prop:etc_linearity_expectation}]
\begin{align*}
    R_T^{(j)} &= \sum_{t=1}^T \mathbb{P}[ I_j(t) \neq k_j^{*}], \\
    &\leq \sum_{i=0}^{\lceil \log_2(T) \rceil +1} \sum_{n = 1}^{2^i} \mathbb{P}[ I_j( \max(T, n + 2^i -1 ) \neq k_j^{*}  ], \\
    &\stackrel{(a)}{\leq} \sum_{i=0}^{\lceil \log_2(T) \rceil +1}\left( \sum_{n=1}^{K\lfloor i^{\varepsilon} \rfloor+N^2}1 + \sum_{n=K\lfloor i^{\varepsilon} \rfloor+N^2+1}^{2^i} \mathbb{P}[\mathcal{E}_i^{\complement}] \right), \\
    &\leq \sum_{i=0}^{\lceil \log_2(T) \rceil +1}( K(i^{\varepsilon}+1) + N^2) + \sum_{i=0}^{\lceil \log_2(T) \rceil +1} 2^i\mathbb{P}[\mathcal{E}_i^{\complement}], \\
    &\leq \int_{i=0}^{\lceil \log_2(T) \rceil +1}(Ki^{\varepsilon}+K+N^2) + \sum_{i=0}^{\lceil \log_2(T) \rceil +1} 2^i\mathbb{P}[\mathcal{E}_i^{\complement}], \\
    &\leq  K\frac{(\log_2(T)+2)^{1+\varepsilon}}{1+\varepsilon} + (N^2+K)( \log_2(T) + 2)  +\sum_{i=0}^{\lceil \log_2(T) \rceil +1} 2^i\mathbb{P}(\mathcal{E}_i^{\complement}).
\end{align*}
In step $(a)$, we use the Gale Shapley property, that if all agents correctly estimate their arm-ranking, then all agents will play the agent-optimal stable match arm in the exploit part of the phase (after accounting for the at-most $N^2$ rounds needed for the Gale-Shapley matching to converge).
\end{proof}

\begin{proof}[Proof of Proposition \ref{prop:etc_chernoff}]
A sufficient condition for the event $\mathcal{E}_i$ to hold, is if all agents, learn of all their respective arm-means, to a resolution of within $\Delta/2$. This will automatically imply that the empirical rankings of all agents will be identical to the truth. For notational simplicity, denote by $\widetilde{\mu}_jk^{(i)}$ to be the empirical mean of arm $k$, computed by agent $j$, using all the explore samples upto and including phase $i$. We have from definition of $\mathcal{E}_i$,
\begin{align*}
    \mathcal{E}_i \supseteq \bigcap_{j=1}^N \bigcap_{k=1}^K \left\{ |\widetilde{\mu}_{jk}^{(i)} - \mu_{jk} | < \frac{\Delta}{2} \ \right\}.
\end{align*}
Thus, we have from the union bound that
\begin{align}
    \mathbb{P}[\mathcal{E}_i^{\complement}] \leq \sum_{j=1}^N \sum_{k=1}^K \mathbb{P} \left[ |\widetilde{\mu}_{jk}^{(i)} - \mu_{jk} | \geq \frac{\Delta}{2} \ \right].
        \label{eqn:etc_proof_pre_final}
\end{align}
Since $\widetilde{\mu}_{jk}^{(i)}$ is the empirical mean estimated using at-least $\frac{(i-1)^{1+\varepsilon}}{1+\varepsilon} - i$ i.i.d. samples, we have from Hoeffding's bound
\begin{align}
    \mathbb{P} \left[ |\widetilde{\mu}_{jk}^{(i)} - \mu_{jk} | \geq \frac{\Delta}{2} \ \right] \leq 2 \exp \left( -2 \left( \frac{(i-1)^{1+\varepsilon}}{1+\varepsilon} - i\right) \frac{\Delta^2}{4}  \right).
    \label{eqn:etc_proof_hoeffding}
\end{align}
The result follows by substituting Equation (\ref{eqn:etc_proof_hoeffding}) into Equation (\ref{eqn:etc_proof_pre_final}).
\end{proof}

\section{Proof of Regret Upper Bound Under SPC Condition}\label{sec:app_spc}
\subsection{Notation and Definition:} We next set up the notations required for the proof of the main results. We denote by $\mathbb{N}$ the set of natural numbers and by $\mathbb{R}_+$ the set of non-negative real numbers.

\textbf{Ranks:} We define by $>_k$ the preference order (a.k.a. rank) of arm $k$, for any arm $k\in [K]$, where if $j>_k j'$ then arm $k$ prefers agent $j$ over agent $j^{'}$. Recall, under the SPC condition we assume the common order (among agents and arms) is identity without loss of generality, for the ease of exposition. 

\textbf{Phases:} Our algorithm works in phases. By $S_i$ we denote the starting round for phase $i$. We have $S_1 = R + 1$ and for $i> 1$, $S_i = R + \sum_{i'=1}^{i-1}( C + 2^{i'-1} )= R + C(i-1) + 2^{i}$, where $R$ is the time required for the Ranking period which runs once in the beginning, and $C$ for the communication phase which is used each phase once. 

\textbf{Arm Classification:} For each agent $j$, let the set of {\em dominated arms} be
$\mathcal{D}_j := \{k^*_{j'}: j'=1 \dots, j-1\}$ the stable matching arm of the agents ranked higher than $j$ in the SPC order. Further, for each arm $k \notin \mathcal{D}_j$, we also define the {\em blocking agents} for arm $k$ and agent $j$ as $\mathcal{B}_{jk} = \{ j': j'>_k j\}$, the set of agents preferred by arm $k$ over agent $j$. We define the arms as {\em hidden arms } $\mathcal{H}_j := \{k: k\notin D_j, \mathcal{B}_{jk}\neq \emptyset\}$.

\textbf{Gaps:} Let the stable matching pair of agent $j$ be arm $k^*_j$ for any $j \in [N]$. Let $\Delta_{jk} = \mu_{jk} - \mu_{jk_j^*}$ be the gap for arm $k \in [K]$. Let $\Delta_{\min} = \min_{jk}\{\Delta_{jk}: k\notin \mathcal{D}_j \cup k^*_j\}$. Recall our assumption that, for every agent, no two arm means are the same implies that $\Delta_{min} > 0$, is strictly greater than $0$.

\textbf{Number of Plays and Attempts:}  For each agent $j \in [N]$ and arm $k\in [K]$, denote by $N_{jk}(t)$ as the number of times agent $j$ has successfully matched with arm $k$ (i.e., without colliding) up to time $t$ for any $t$. For all $i\geq 1$, we denote by $N_{jk}[i] = N_{jk}(S_{i+1}-1)$ to be the same quantity as above at the end of phase $i$. We denote by $I_j(t) \in [K]\cup {\emptyset}$ as the arm sampled by agent $j$ on round $t$ (here $I_j(t)=\emptyset$ denotes the agent $j$ collides in time $t$). Let $G_j[i]$ denote the set of globally deleted arms for agent $j$ at the beginning of phase $i$, and $L_j[i]$ denote the set of locally deleted arms for agent $j$ at the end of phase $i$. 
We note that $N_{jk}$, $I_j(t-1)$, $G_j[i]$, and $L_j[i]$ all are random variables adapted to the filtration constructed by the history of agent $j$ at different points in time. 

\textbf{Critical Phases:} 
\begin{itemize}
    \item The phase $i$ for agent $j$, for some $j \in [N]$, is a {\em Good Phase} if the following are true:
        \begin{itemize}
            \item[1.] The dominated arms are globally deleted, i.e. $G_j[i] = \mathcal{D}_j$.
            \item[2.] For each arm $k\notin \mathcal{D}_j \cup k^*_j$ (not globally deleted), in phase $i$ arm $k$ is successfully played (a.k.a. sampled) by agent $j$ at most $\tfrac{10\gamma i}{\Delta_{jk}^2}$ times.  
            \item[3.] The stable match pair arm $k_j^*$ is sampled the most number of times in phase $i$. 
        \end{itemize}
    The {\em good phase} definition is identical to the definition in Sankararaman et al.~\cite{sankararaman2020dominate}.
    
    
    \item We further define a phase $i$ for agent $j$, for some $j \in [N]$ to be a {\em Low Collision Phase} if the following are true:
        \begin{itemize}
            \item[1.] Phase $i$ is a good phase for agent $1$ to $j$.
            \item[2.] Phase $i$ is a good phase for all agent $j' \in \cup_{k\in \mathcal{H}_j} \mathcal{B}_{jk}$. 
        \end{itemize}
\end{itemize}

For notational ease, let  $\mathbb{I}_G[i,j]$ be the indicator that phase $i$ is a good phase for agent $j$. Similarly, let  $\mathbb{I}_{LC}[i,j]$ be the indicator that phase $i$ is a low collision phase for agent $j$. 

Let $i_1=\left((N-1)\tfrac{10\gamma}{\Delta^2_{min}}\right)^{\tfrac{1}{\beta-1}}+1$.
We define the {\em Freezing} phase for each agent $j$ as the phase on and after which all phases are good phases for agents $1$ to $(j-1)$.
$$F_j = \max\left(i_1, \min \left(\{i: \prod_{j'=1}^{(j-1)}\prod_{i'\geq i}\mathbb{I}_G[i',j'] = 1 \} \cup \{\infty\}\right)\right).$$  
Further, the {\em Vanishing} phase for each agent $j$  as the phase on and after which all phases are low collision phases for agent $j$
$$V_j = \max\left(i_1, \min\left(\{i: \prod_{i'\geq i}\mathbb{I}_{LC}[i',j] = 1\} \cup \{\infty\}\right)\right).$$ 

It is easy to see, that $F_j = \max\left(F_{j'}: 1\leq j'\leq (j-1)\right)$ and $V_j = \max\left(F_{j+1}, \cup_{k\in \mathcal{H}_j}\cup_{j'\in \mathcal{B}_{jk}} F_{j'} \right)$ from the definition of low collision phase.

\subsection{Proof of main result} 
We begin the proof with the simple result that arm $j$ and agent $j$ form a stable match pair for any $j\leq N$
\begin{proposition}\label{prop:stable}
If a system satisfies SPC then $k^*_j=j$ for all $j\in [N]$. Furthermore, if a system satisfies $\alpha$-condition then we have $k^*_j=j$ for all $1\leq j \leq N$, and $j^*_{a_k} = A_k$ and $k^*_{A_j}= a_j$ for all $ 1\leq k,j \leq N$.
\end{proposition}
\begin{proof}
We note for the SPC condition, for $j=1$ we have $\mu_{1k^*_1} > \mu_{1k}$ for all $k > 1$ implying $k^*_1 = 1$. For $j=2$ we see $\mu_{2k^*_2} > \mu_{2k}$ for all $k > 2$. So $k^*_2 \in \{1,2\}$. But $k^*_1=1$, thus $k^*_2=2$. This logic can be extended to prove that for all $1\leq j\leq N$ $k^*_j=j$.
\end{proof}

We now prove that under low collision phase there is no local deletion for agent $j$ in the following lemma. 
\begin{lemma}\label{lemm:structural}
If a phase $i \geq i_1=\min\{i: (N-1)\tfrac{10\gamma i}{\Delta^2_{min}} < \beta 2^{(i-1)}\}$, is a Low Collision phase for agent $j$, for any $j\in [N]$, then $L_j[i]=\emptyset$.
\end{lemma}
\begin{proof}
As phase $i$ is Low Collision we know  phase $i$ is good for the agent $j$. Therefore, the arms $k\in \mathcal{D}_j$ are deleted in the beginning of the phase and no collision is encountered. So $\forall k\in \mathcal{D}_j, k \notin L_j[i]$.

Further, as phase $i$ is a good phase for all agent $j' \in \cup_{k\in \mathcal{H}_j} \mathcal{B}_{jk}$ (by definition), for each $k\in \mathcal{H}_j$ the maximum number of collision $(N-1)\tfrac{10\gamma i}{\Delta^2_{min}}$. This is true as in a good phase  any agent $j' \in \mathcal{B}_{jk}$ plays $k$ at most $\tfrac{10\gamma i}{\Delta^2_{min}}$ times as $k \notin \mathcal{D}_{j'} \cup k^*_{j'}$. Therefore, for $i \geq  i_1$ we have local deletion threshold $\beta 2^{(i-1)}> (N-1)\tfrac{10\gamma i}{\Delta^2_{min}}$. Thus, $\forall k \in \mathcal{H}_j,  k \notin L_j[i]$. 

Finally, for each arm $k\notin \mathcal{H}_j \cup \mathcal{D}_j$, we have for all agent $j'\neq j, j'>_k j$. Therefore,  $\forall k \notin \mathcal{H}_j \cup \mathcal{D}_j$, there is no collision experienced by agent $j$. Therefore,  $k \notin L_j[i]$ and we conclude that $L_j[i] = \emptyset$.
\end{proof}

\begin{lemma}\label{lemm:structural2}
If a phase $i \geq i_1$ ($i_1$ as defined in Lemma~\ref{lemm:structural}), is a good phase for all agents $1$ to $(j-1)$, for any $j\in [N]$, then $k^*_j \notin L_j[i]$.
\end{lemma}
\begin{proof}
Due to SPC we have $j >_{k^*_j} j'$ for all $j'>j$. Therefore, the arm $k^*_j$ can be locally deleted ($k^*_j \in L_j[i]$) in phase $i$ only if the total collisions from agents $1$ to $(j-1)$ is greater than $i^\beta$. But the total collisions from agents $1$ to $(j-1)$ is at most $\tfrac{10 (j-1)\gamma i}{\Delta^2_{min}}$ as phase $i$ is a good phase for all agents $1$ to $(j-1)$. Also, $i^\beta > (N-1)\tfrac{10\gamma i}{\Delta^2_{min}}$ for all $i \geq i_1$. Therefore, $k^*_j$ can not be locally deleted in the phase $i$ as mentioned in the lemma.
\end{proof}

We now decompose the regret and provide a regret upper bound.
\begin{lemma}\label{lemm:decompose}
The expected regret for agent $j$ can be upper bounded as  
\begin{align*}
    \mathbb{E}[R_j(T)] &\leq \mathbb{E}[S_{F_j}] + \min(1, \beta|\mathcal{H}_j|)\mathbb{E}[S_{V_j}] + \left((K-1+|\mathcal{B}_{jk^*_j}|)\log_2(T) + NK \mathbb{E}[V_{j}]\right)\\
    & + \sum_{k\notin \mathcal{D}_j} \sum_{j'\in \mathcal{B}_{jk}: k \notin \mathcal{D}_{j'}} \frac{8\gamma \mu_{k^*_j}}{\Delta^2_{j'k}}\left(\log(T)+\sqrt{\tfrac{\pi}{\gamma}\log(T)}\right)
    + \sum_{k\notin \mathcal{D}_j \cup k^*_j}\frac{8\gamma}{\Delta_{jk}}\left(\log(T)+\sqrt{\tfrac{\pi}{\gamma}\log(T)}\right) \\
    & + NK \left(1+(\psi(\gamma)+1)\frac{8\gamma}{\Delta^2_{\min}}\right)
\end{align*}
\end{lemma}
\begin{proof}
We now decompose the expected regret as follows 
\begin{align*}
    &\mathbb{E}[R_j(T)] \\
    &\leq \mathbb{E}[S_{F_j}] + \mathbb{E}\left[\sum_{i=(F_j+1)}^{V_j} \sum_{k\notin \mathcal{D}_j} \beta2^{i-1}\right] +  \left(c_j\log_2(T) + NK \mathbb{E}[F_{(j+1)}]\right)\\
    &+ \mathbb{E}\left[\sum_{k\notin \mathcal{D}_j} \sum_{j'\in \mathcal{B}_{jk}: k \notin \mathcal{D}_{j'}} \mu_{k^*_j}(N_{j'k}(T) - N_{j'k}(S_{V_j}))\right]\\ 
    &+ \mathbb{E}\left[\sum_{k\notin \mathcal{D}_j \cup k^*_j}\Delta_{jk}(N_{jk}(T) - N_{jk}(S_{F_j}))\right]\\
    &\leq \mathbb{E}[S_{F_j}] + \min(1, \beta|\mathcal{H}_j|)\mathbb{E}\left[S_{V_j}\right] + \left(c_j\log_2(T) + NK \mathbb{E}[F_{(j+1)}]\right) \\
    &+ \sum_{k\notin \mathcal{D}_j} \sum_{j'\in \mathcal{B}_{jk}: k \notin \mathcal{D}_{j'}} \mu_{k^*_j}\mathbb{E}\left[(N_{j'k}(T) - N_{j'k}(S_{V_j}))\right] \\ 
    &+ \sum_{k\notin \mathcal{D}_j \cup k^*_j}\Delta_{jk}\mathbb{E}\left[(N_{jk}(T) - N_{jk}(S_{F_j}))\right]\\
    &\leq \mathbb{E}[S_{F_j}] + \min(1, \beta|\mathcal{H}_j|)\mathbb{E}\left[S_{V_j}\right] + \left(c_j\log_2(T) + NK \mathbb{E}[F_{(j+1)}]\right)\\ 
    &+ \sum_{k\notin \mathcal{D}_j} \sum_{j'\in \mathcal{B}_{jk}: k \notin \mathcal{D}_{j'}} \mu_{k^*_j}\left( \psi(\gamma)\tfrac{8}{\Delta_{j'k}^2} + 1 + \tfrac{8}{\Delta_{j'k}^2} \left(\gamma\log(T) + \sqrt{\pi\gamma\log(T)} + 1 \right)\right)\\ 
    &+ \sum_{k\notin \mathcal{D}_j \cup k^*_j}\Delta_{jk}\left(\psi(\gamma)\tfrac{8}{\Delta_{jk}^2} + 1 + \tfrac{8}{\Delta_{jk}^2} \left(\gamma\log(T) + \sqrt{\pi\gamma\log(T)} + 1 \right)\right).
\end{align*}
In the first inequality, follows due to the following reasons. 
\begin{itemize}[noitemsep]
    \item We upper bound the regret till the end of phase $F_j$ by $S_{F_j}$ as regret per round is at most $\mu_{k^*_j}\leq 1$.
    
    \item {\bf Local deletion:} Next from phase $(F_j+1)$ upto phase $V_j$ (both inclusive), we upper bound the regret due to collision by $\sum_{i=(F_j+1)}^{V_j} \sum_{k\in \mathcal{H}_j} \beta 2^{(i-1)}$ as in each round $i$ at most $\beta2^{(i-1)}$ collisions are possible when pulling an arm from the set $|\mathcal{H}_j|$ in phase $(F_j+1)$ to $V_j$. This is true as all the arms in $\mathcal{D}_j$ are globally deleted from phase $(F_j+1)$ onwards.
    
    \item {\bf Communication phase:} The best arm for agent $j$ is not played in all but $(K-1)$ number of steps for each communication phase after phase $F_{j+1}$, and other agents $j'\in \mathcal{B}_{jk^*_j}$ collide at most once after phase $V_{j}$ (as each of them enter good phase). Thus, beyond phase $V_{j}$ we have at most $(K-1+|\mathcal{B}_{jk^*_j}|)$ regret due to collision,  and there are at most $\log_2(T)$ communication phases. This limits the regret due to communicationat $\left((K-1+|\mathcal{B}_{jk^*_j}|)\log_2(T) + NK V_{j}\right)$. 
    
    \item {\bf Collision:} From phase $(V_j+1)$ (inclusive) onwards only an agent $j' \in \mathcal{B}_{jk}$ collides with $k$ only if $k\notin \mathcal{D}_{j'}$, because (1) agent $j'$ deletes all arms in $\mathcal{D}_{j'}$ from $(V_j+1)$ (inclusive) onwards, and (2) all $j' \notin \mathcal{B}_{jk}$ and $j'\neq j$, $j>_k j'$ (agent $j$ is preferred by $k$ over agent $j'$). This amounts to $\sum_{k\notin \mathcal{D}_j} \sum_{j'\in \mathcal{B}_{jk}: k \notin \mathcal{D}_{j'}} \mu_{k^*_j}(N_{j'k}(T) - N_{j'k}(S_{V_j}))$ regret.
    
    \item  {\bf Suboptimal play:} Finally, from phase $(F_j+1)$ (inclusive) onwards till the last phase agent $j$ incurs regret $\Delta_{jk}$ each time the agent $j$ successfully plays arm $k \notin \mathcal{D}_j \cup k^*_j$. Thus she incurs total  $\Delta_{jk}(N_{jk}(T) - N_{jk}(S_{F_j}))$ regret for each such $k$.
\end{itemize}
The validity of the second inequality is easy to see. We now come to the last inequality. We know that $(N_{j'k}(T) - N_{j'k}(S_{V_j})) \leq (N_{j'k}(T) - N_{j'k}(S_{F_{j'}}))$ (almost surely) as $V_j \geq F_{j'}$ almost surely from the definition of $V_j$, for all $j'\in \mathcal{B}_{jk}$. Thus, the final inequality follows by substituting the bounds from Lemma~\ref{lemm:bound}.
\end{proof}

We now prove the upper bound on the expected number of times a sub-optimal arm is played by an agent $j$ after the Global deletion freezes. 

\begin{lemma}\label{lemm:bound}
For any $j \in [N]$, $k\notin \mathcal{D}_j\cup k^*_j$, for $\gamma > 1$,
$$
\mathbb{E}\left[(N_{jk}(T) - N_{jk}(S_{F_j}))\right]
\leq \psi(\gamma)\tfrac{8}{\Delta_{jk}^2} + 1 + \tfrac{8}{\Delta_{jk}^2} \left(\gamma\log(T) + \sqrt{\pi\gamma\log(T)} + 1 \right). 
$$
\end{lemma}
\begin{proof}
We have for any $k \notin \mathcal{D}_j\cup k^*_j$ and $\epsilon > 0$
\vspace{-1em}
\begin{align}
    &(N_{jk}(T) - N_{jk}(S_{F_j})) = \sum_{t=S_{F_j}+1}^{T} \mathbb{I}(I_j(t)=k)\notag\\
    &\leq \sum_{t=S_{F_j}+1}^{T} \left(\mathbb{I}(u_{jk}(t-1) \geq \mu_{jk^*_j} - \epsilon \wedge I_j(t)=k ) + \mathbb{I}(u_{jk^*_j}(t-1) \leq \mu_{jk^*_j} - \epsilon)\right) \label{eq:spckey}
\end{align}
The inequality is true because phase $(F_j+1)$ onwards (inclusive) the arm $k^*_j$ is neither globally deleted (by definition of $F_j$) or locally deleted as shown in Lemma~\ref{lemm:structural2}. Therefore, we have
$$\{I_j(t) = k \wedge t > S_{F_j}\} \subseteq \{ I_j(t) = k \wedge u_{jk}(t-1) \geq \mu_{jk^*_j} - \epsilon  \wedge t > S_{F_j}\} \cup \{u_{jk^*_j}(t-1) \leq \mu_{jk^*_j} - \epsilon \wedge t > S_{F_j}\}.$$
Note, before the phase $(F_j+1)$ it is not true that any arm $k'$ better than $k$, in particular arm $k^*_j$, survives the global and local deletion. The rest of the proof of this lemma is fairly standard. However, we present it for completeness. 

We next bound the expectation of the second term in a standard way as follows (c.f. \cite{lattimore2020bandit} Theorem 8.1 proof)
\begin{align*}
    &\mathbb{E}[\sum_{t=S_{F_j}+1}^{T} \mathbb{I}(u_{jk^*_j}(t-1) \leq \mu_{jk^*_j} - \epsilon)]\\ 
    &= \mathbb{E}[\sum_{t=1}^{T} \mathbb{I}(u_{jk^*_j}(t-1) \leq \mu_{jk^*_j} - \epsilon \wedge t > S_{F_j})]\\
    &\leq \mathbb{E}[\sum_{t=1}^{T} \mathbb{I}(u_{jk^*_j}(t-1) \leq \mu_{jk^*_j} - \epsilon)]\\
    &\leq \sum_{t=1}^{T} \sum_{s=1}^{T}\mathbb{P}(\hat{\mu}_{jk^*_j,s} + \sqrt{\tfrac{2\gamma \log(t)}{s}} \leq \mu_{jk^*_j} - \epsilon)\\
    &\leq \sum_{t=1}^{T} \sum_{s=1}^{T} \exp\left(-\tfrac{s}{2} (\sqrt{\tfrac{2\gamma \log(t)}{s}}+\epsilon)^2\right)\\
    &\leq \sum_{t=1}^{T} t^{-\gamma}\sum_{s=1}^{T} \exp(-\tfrac{s \epsilon^2}{2}) \leq \psi(\gamma)\tfrac{2}{\epsilon^2}
\end{align*}
Here, $\psi()$ is the Riemann zeta function. 
Note, the first inequality is valid as $$\mathbb{I}(u_{jk^*_j}(t-1) \geq \mu_{jk^*_j} - \epsilon \wedge t > S_{F_j}) \leq \mathbb{I}(u_{jk^*_j}(t-1) \geq \mu_{jk^*_j} - \epsilon) \text{ a.s. }$$

Finally, we bound the expectation of the first term also in a standard way ((c.f. \cite{lattimore2020bandit} Lemma 8.2)) 
\begin{align*}
  &\mathbb{E}\left[\sum_{t=S_{F_j}+1}^{T} \mathbb{I}(u_{jk}(t-1) \geq \mu_{jk^*_j} - \epsilon \wedge I_j(t)=k)\right]\\
  & \leq \mathbb{E}\left[\sum_{t=1}^{T} \mathbb{I}(\hat{\mu}_{jk}(t-1) + \sqrt{\tfrac{2\gamma \log(t)}{N_{jk}(t-1)}} \geq \mu_{jk^*_j} - \epsilon \wedge I_j(t)=k) \right]\\
  & \leq \mathbb{E}\left[\sum_{t=1}^{T} \mathbb{I}(\hat{\mu}_{jk}(t-1) + \sqrt{\tfrac{2\gamma \log(T)}{N_{jk}(t-1)}} \geq \mu_{jk^*_j} - \epsilon \wedge I_j(t)=k)\right]\\
  & \leq \mathbb{E}\left[\sum_{s=1}^{T} \mathbb{I}(\hat{\mu}_{jk,s} + \sqrt{\tfrac{2\gamma \log(T)}{s}} \geq \mu_{jk}+\Delta_{jk} - \epsilon)\right] \\
  & \leq 1+ \frac{2}{(\Delta_{jk}-\epsilon)^2} \left(\gamma\log(T) + \sqrt{\pi\gamma\log(T)} + 1 \right) 
\end{align*}
We combine the two above bounds and pick $\epsilon =\Delta_{jk}/2$ (for simplicity, this can be tightened with some effort) we obtain the following bound for $\gamma > 1$,
$$
\mathbb{E}\left[(N_{jk}(T) - N_{jk}(S_{F_j}))\right]
\leq \psi(\gamma)\tfrac{8}{\Delta_{jk}^2} + 1 + \tfrac{8}{\Delta_{jk}^2} \left(\gamma\log(T) + \sqrt{\pi\gamma\log(T)} + 1 \right). 
$$
\end{proof}

We now have to provide an upper bound on the moments of  $V_j$ and mean of $S_{F_j}$ to complete the proof of the regret bound. As $V_j$ is a function of $F_{j'}$ for $j'\in [N]$ we need to derive bounds for moments and exponent of $F_{j'}$ for all $j'$.  The key idea is to show that once the Global deletion has settled for agents $1$ to $(j-1)$ (recall the agents are  ordered according to the SPC order) the agent $j$ enters Good phase with high probability. 
\begin{lemma}\label{lemm:goodphase}
For any agent $j$ and any phase $i \geq i^* =  \max\{8, i_1, i_2\}$ and $\gamma > 1$, 
$$\mathbb{P}[\mathbb{I}_G[i,j] = 0 \wedge i\geq F_{j}+1]\leq (K - 1 - |\mathcal{D}_j|) \left( 1 + \tfrac{64}{\Delta_{\min}^2}\right) 2^{-i(\gamma -1)},$$ 
where $i_1 = \min\{i: (N-1)\tfrac{10\gamma i}{\Delta^2_{min}} < \beta 2^{(i-1)}\}$ and $i_2 = \min \{i: (R-1 + C(i-1))\leq 2^{i+1}\}$.
\end{lemma}
\begin{proof}
Let us recall that the phase $i$ is a Good phase for agent $j$ if and only if (1) the dominated arms $\mathcal{D}_j$ are deleted in global deletion, and (2) each arm $k\notin \mathcal{D}_j \cup k^*_j$ is sampled by agent $j$ at most $\tfrac{10\gamma i}{\Delta_{jk}^2}$ times, and (3) arm $k_j^*$ is matched the most number of times.

Given $\{i\geq (F_{j}+1)\}$ in phase $i$ condition (1) is satisfied for any $i\geq 1$. For $i\geq i_1$, as in Lemma~\ref{lemm:structural}, we can show that condition (1) and (2) implies condition (3) holds true. So we need to bound the probability that given  $\{i\geq F_{j}+1\}$ the condition (2) holds. This follows from the properties of UCB as we show below. We have for any $j$ and $\epsilon > 0$,
\begin{align*}
    &\mathbb{P}[\mathbb{I}_G[i,j] = 0\wedge i\geq (F_j+1)]\\
    &\leq \mathbb{P}\left[ \cup_{k\notin  \mathcal{D}_j \cup k^*_j} \{(N_{jk}[i] - N_{jk}[i-1]) > \tfrac{10\gamma i}{\Delta_{jk}^2} \} \wedge i\geq (F_j+1)\right]\\
    &\leq \sum_{k\notin  \mathcal{D}_j \cup k^*_j}\mathbb{P}\left[ \cup_{t\in S_{i}}^{(S_{i+1}-1)} N_{jk}(t) = \tfrac{10\gamma i}{\Delta_{jk}^2} \wedge I_{j}(t) = k\wedge i\geq (F_j+1)\right]\\
    &\stackrel{(i)}{\leq}  \sum_{k\notin  \mathcal{D}_j \cup k^*_j} \sum_{t\in S_{i}}^{(S_{i+1}-1)} \mathbb{P}\left[ N_{jk}(t) = \tfrac{10\gamma i}{\Delta_{jk}^2} \wedge u_{jk}(t-1) > u_{jk^*_j}(t-1)\right]\\
    &\stackrel{(ii)}{\leq} \sum_{k\notin  \mathcal{D}_j \cup k^*_j} \sum_{t\in S_{i}}^{(S_{i+1}-1)}  \mathbb{P}\left[\{N_{jk}(t) = \tfrac{10\gamma i}{\Delta_{jk}^2} \wedge u_{jk}(t-1) \geq \mu_{jk^*_j} - \epsilon \} \cup \{u_{jk^*_j}(t-1) \leq \mu_{jk^*_j}  - \epsilon\} \right]\\
    &\stackrel{(iii)}{\leq} \sum_{k\notin  \mathcal{D}_j \cup k^*_j} \sum_{t\in S_{i}}^{(S_{i+1}-1)}  \mathbb{P}\left[N_{jk}(t) = \tfrac{10\gamma i}{\Delta_{jk}^2} \wedge \hat{\mu}_{jk}(t-1)  +  \Delta_{jk}\sqrt{\tfrac{2\gamma \log(t)}{10\gamma i}} \geq \mu_{jk} + \Delta_{jk} - \epsilon \right] + \\
    & + \sum_{k\notin  \mathcal{D}_j \cup k^*_j} \sum_{t\in S_{i}}^{(S_{i+1}-1)}  \sum_{s=1}^{t-1} \mathbb{P}\left[ \hat{\mu}_{jk^*_j,s} + \sqrt{\tfrac{2\gamma \log(t)}{s}} \leq \mu_{jk^*_j}  - \epsilon\right]\\
    &\stackrel{(iv)}{\leq} \sum_{k\notin  \mathcal{D}_j \cup k^*_j} \sum_{t\in S_{i}}^{(S_{i+1}-1)} \left(\exp\left(-\tfrac{5\gamma i}{\Delta^2_{jk}} (\tfrac{1}{2}\Delta_{jk}-\epsilon)^2\right) +  \sum_{s=1}^{t-1} \exp\left(-\tfrac{s}{2} (\sqrt{\tfrac{2\gamma \log(t)}{s}}+\epsilon)^2\right)\right)\\
    &\stackrel{(v)}{\leq} \sum_{k\notin  \mathcal{D}_j \cup k^*_j} (S_{i+1}- S_i) 2^{-i\gamma} \left( 1 + \tfrac{64}{\Delta_{jk}^2}\right) \\
    &\stackrel{(vi)}{\leq} (K - j)2^{-i(\gamma -1)} \left( 1 + \tfrac{64}{\Delta_{\min}^2}\right)
\end{align*}

Inequality (i) relates the event of playing arm $k$ to the UCB bounds along with the fact that $k^*_j$ is present after phase $i\geq (F_j+1)$. The inequalities (ii) and (iii) follow similar logic as in  Lemma~\ref{lemm:bound}. 
Here for inequality  (iv) we use large enough $i$  such that $(1-\sqrt{\tfrac{\log(S_{i+1}-1)}{5i}}) \geq 1/2$, and for (v) we use small enough $\epsilon$ such that $\tfrac{5}{\log(2)} (\tfrac{1}{2}-\tfrac{\epsilon}{\Delta_{jk}})^2 \geq 1$. We also use the fact $S_i \geq 2^i$ for (v).  The above are satisfied when $\epsilon = \Delta_{jk} / 8$ and for all $i \geq \max\{8, i_2\}$. The latter is true because for $i_2 = \min \{i: (R-1 + C(i-1))\leq 2^{i+1}\}$ we have $\log(S_{i+1}-1)\leq i+2$, and for $i \geq 8$, $ 1- \sqrt{\tfrac{i+2}{5i}}\geq 1/2$. Finally, (vi) simply uses minimum gap over all arms and agents (for simplicity) and $|\mathcal{D}_j|=(j-1)$.
\end{proof}

To complete the proof we need to upper bound of the moments, and exponents of $F_j$ in an inductive manner similar to Sankararaman et al.~\cite{sankararaman2020dominate}.
\begin{lemma}\label{lemm:moments}
For any $j\in [N]$ and $m\geq 1$, the following hold with $i^*$ as defined in Lemma~\ref{lemm:goodphase}
\begin{align*}
    &\mathbb{E}[F_j^m] \leq i_1 + (j-1) (i^*)^m+ (j-1)(K  - j/2) \left( 1 + \tfrac{64}{\Delta_{\min}^2}\right)  \tfrac{2^{- (\gamma -1)(i^*-2)}}{(2^{(\gamma-1)}-1)^2}\\
    & \mathbb{E}[2^{F_j}] \leq i_1 + (j-1)2^{i^*}+ (j-1)(K  - j/2) \left( 1 + \tfrac{64}{\Delta_{\min}^2}\right)  \tfrac{2^{- (\gamma -1)(i^*-2)}}{(2^{(\gamma-1)}-1)^2}.
\end{align*}
\end{lemma}
\begin{proof}
Let $g:\mathbb{R} \to \mathbb{R}_+$ be any monotonically increasing and continuous (hence invertible) function. We have that $F_0 = i_1$ almost surely by definition (this accounts for the max with $i_1$ in the definitino of $F_j$). The inductive hypothesis is 
$$
\mathbb{E}[g(F_j)] \leq i_1 + (j-1)g(i^*)+ (j-1)(K  - j/2) \left( 1 + \tfrac{64}{\Delta_{\min}^2}\right)  \tfrac{2^{- (\gamma -1)(i^*-2)}}{(2^{(\gamma-1)}-1)^2}.
$$ 
We calculate the expectation for agent $j$ as  
\begin{align*}
    \mathbb{E}[g(F_j)] &= \sum_{x\geq 0} \mathbb{P}[g(F_j) \geq x]\\
    &=  \sum_{x\geq 0} \mathbb{P}[F_j \geq g^{-1}(x)]\\
    &= \sum_{x\geq 0} \left(\mathbb{P}[F_j \geq g^{-1}(x), F_{j-1} \geq g^{-1}(x)] +
    \mathbb{P}[F_j \geq g^{-1}(x), F_{j-1} < g^{-1}(x)]\right)\\
    &\leq \sum_{x\geq 0} \mathbb{P}[F_{j-1} \geq g^{-1}(x)] +
    \sum_{x\geq 0}\mathbb{P}[F_j \geq g^{-1}(x), F_{j-1} < g^{-1}(x)]\\
    &\leq \mathbb{E}[g(F_{j-1})] + \sum_{x=0}^{g(i^*)-1} 1 + 
    \sum_{x\geq g(i^*)} \mathbb{P}[F_j \geq g^{-1}(x), F_{j-1} < g^{-1}(x)]\\
    &\leq \mathbb{E}[g(F_{j-1})] + g(i^*) + 
    \sum_{i\geq i^*} \mathbb{P}[F_j \geq i, F_{j-1} < i]\\
    &\leq \mathbb{E}[g(F_{j-1})] + g(i^*) + \sum_{i\geq i^*} \mathbb{P}[\{\exists i'\geq i, \mathbb{I}_G[i',j]=0\}, F_{j-1}+1 \leq i]\\
    &\leq \mathbb{E}[g(F_{j-1})] + g(i^*) + \sum_{i\geq i^*} \sum_{i'\geq i}\mathbb{P}[\mathbb{I}_G[i',j]=0, F_{j-1}+1 \leq i']\\
    &\leq \mathbb{E}[g(F_{j-1})] + g(i^*) + \sum_{i'\geq i^*} (i'-i^*+1)\mathbb{P}[\mathbb{I}_G[i',j]=0, F_{j-1}+1 \leq i']\\
    &\stackrel{(i)}{\leq} \mathbb{E}[g(F_{j-1})] + g(i^*)  +  (K - j) \left( 1 + \tfrac{64}{\Delta_{\min}^2}\right) \sum_{i'\geq i^*} (i'-i^*+1) 2^{-i'(\gamma -1)}\\
    &\leq (j-2)g(i^*)+ (j-2)(K  - j/2 + 1/2) \left( 1 + \tfrac{64}{\Delta_{\min}^2}\right)  \tfrac{2^{- (\gamma -1)(i^*-2)}}{(2^{(\gamma-1)}-1)^2}\\
    & + g(i^*) + (K - j) \left( 1 + \tfrac{64}{\Delta_{\min}^2}\right) \tfrac{2^{- (\gamma -1)(i^*-2)}}{(2^{(\gamma-1)}-1)^2}\\
    &\leq i_1 + (j-1) g(i^*)+ (j-1)(K  - j/2) \left( 1 + \tfrac{64}{\Delta_{\min}^2}\right)  \tfrac{2^{- (\gamma -1)(i^*-2)}}{(2^{(\gamma-1)}-1)^2}.
\end{align*}
The inequality (i) follows due to Lemma~\ref{lemm:goodphase}, while the rest are standard. 
\end{proof}

To finalize the regret upper bound proof we note that the following holds. For the expected rounds upto the end of phase $F_j$ is upper bounded as 
\begin{align*}
&\mathbb{E}[S_{F_j}] = \mathbb{E}[ R + C(F_j-1) + 2^{F_j}] \\
&\leq R + C(i_1 + (j-1) i^* + (j-1)(K  - j/2) \left( 1 + \tfrac{64}{\Delta_{\min}^2}\right)  \tfrac{2^{- (\gamma -1)(i^*-2)}}{(2^{(\gamma-1)}-1)^2} -1) \\
&+ (j-1)2^{i^*}+ (j-1)(K  - j/2) \left( 1 + \tfrac{64}{\Delta_{\min}^2}\right)  \tfrac{2^{- (\gamma -1)(i^*-2)}}{(2^{(\gamma-1)}-1)^2}\\
&= R + C(i_1-1) + C(j-1) i^* + (j-1)2^{i^*} + (C+1)(j-1)(K  - j/2) \left( 1 + \tfrac{64}{\Delta_{\min}^2}\right)  \tfrac{2^{- (\gamma -1)(i^*-2)}}{(2^{(\gamma-1)}-1)^2}
\end{align*}
Let us define $J_{\max}(j) = \max\left(j+1, \{j': \exists k\in \mathcal{H}_j, j'\in \mathcal{B}_{jk}\}\right)$.
Then as $F_j \geq F_{j'}$ almost surely for all $j \geq j'$ by definition, we have  
$$V_j = \max\left(F_{j+1}, \cup_{k\in \mathcal{H}_j}\cup_{j'\in \mathcal{B}_{jk}} F_{j'} \right) = F_{J_{\max}(j)}.$$ 
Thus, for to upper bound the regret upto the end of the phase when the local deletion vanishes is given as 
$\mathbb{E}[S_{V_j}] \leq \mathbb{E}[S_{F_{J_{\max}(j)}}]$.
Combining the above two inequalities with the result in Lemma~\ref{lemm:bound} we obtain the final regret bound as
\begin{align*}
    &\mathbb{E}[R_j(T)]\\ 
    &\leq R + C(i_1-1) + C(j-1) i^* + (j-1)2^{i^*} + (C+1)(j-1)K  \left( 1 + \tfrac{64}{\Delta_{\min}^2}\right)  \tfrac{2^{- (\gamma -1)(i^*-2)}}{(2^{(\gamma-1)}-1)^2} 
    + \min(1, \beta|\mathcal{H}_j|) \times \dots \\
    & \dots\times\left(R + C(i_1-1) + C(J_{\max}(j)-1) i^* + (J_{\max}(j)-1)2^{i^*} + (C+1)(J_{\max}(j)-1)K  \left( 1 + \tfrac{64}{\Delta_{\min}^2}\right)  \tfrac{2^{- (\gamma -1)(i^*-2)}}{(2^{(\gamma-1)}-1)^2}\right) \\
    &+ (K-1+|\mathcal{B}_{jk^*_j}|)\log_2(T) + NK \left( i_1+ (j-1) i^* + (J_{\max}(j)-1)K \left( 1 + \tfrac{64}{\Delta_{\min}^2}\right)  \tfrac{2^{- (\gamma -1)(i^*-2)}}{(2^{(\gamma-1)}-1)^2}\right)\\
    & + \sum_{k\notin \mathcal{D}_j} \sum_{j'\in \mathcal{B}_{jk}: k \notin \mathcal{D}_{j'}} \tfrac{8\gamma \mu_{k^*_j}}{\Delta^2_{j'k}}\left(\log(T)+\sqrt{\tfrac{\pi}{\gamma}\log(T)}\right)
    + \sum_{k\notin \mathcal{D}_j \cup k^*_j}\tfrac{8\gamma}{\Delta_{jk}}\left(\log(T)+\sqrt{\tfrac{\pi}{\gamma}\log(T)}\right)  \\
    &+ NK \left(1+(\psi(\gamma)+1)\frac{8\gamma}{\Delta^2_{\min}}\right)\\
    &\leq \sum_{k\notin \mathcal{D}_j} \sum_{j'\in \mathcal{B}_{jk}: k \notin \mathcal{D}_{j'}} \tfrac{8\gamma \mu_{k^*_j}}{\Delta^2_{j'k}}\left(\log(T)+\sqrt{\tfrac{\pi}{\gamma}\log(T)}\right) + \sum_{k\notin \mathcal{D}_j \cup k^*_j}\tfrac{8\gamma}{\Delta_{jk}}\left(\log(T)+\sqrt{\tfrac{\pi}{\gamma}\log(T)}\right) \\
    &+ (K-1+|\mathcal{B}_{jk^*_j}|)\log_2(T)  + O\left(\tfrac{N^2K^2}{\Delta^2_{\min}} + (\min(1, \beta |\mathcal{H}|_j)J_{\max}(j)+ j-1)2^{i^*}  + N^2K i^*\right)
\end{align*}

This gives us the following rerget bound under the SPC setting 
\begin{theorem}\label{thm:spc}
For a stable matching instance satisfying $\alpha$-condition (Definition~\ref{def:alpha}), suppose each agent follows $\mainAlg$ (Algorithm~\ref{alg:UCBD4}) with $\gamma > 1$ and $\beta \in (0,1/K)$, then the regret for an agent $j\in [N]$ is upper bounded by 
\begin{align*}
    &\mathbb{E}[R_j(T)] \leq \underbrace{\sum_{k\notin \mathcal{D}_j \cup k^*_j}\tfrac{8\gamma}{\Delta_{jk}}\left(\log(T)+\sqrt{\tfrac{\pi}{\gamma}\log(T)}\right)}_{\text{sub-optimal match}}
    + \underbrace{\sum_{k\notin \mathcal{D}_j} \sum_{j'\in \mathcal{B}_{jk}: k \notin \mathcal{D}_{j'}} \tfrac{8\gamma \mu_{k^*_j}}{\Delta^2_{j'k}}\left(\log(T)+\sqrt{\tfrac{\pi}{\gamma}\log(T)}\right)}_{\text{collision}} \\
    &+ \underbrace{(K - 1 + |\mathcal{B}_{jk^*_j}|)\log_2(T)}_{\text{communication}} + 
    + \underbrace{O\left(\tfrac{N^2K^2}{\Delta^2_{\min}} + ( \beta|\mathcal{H}_j| J_{\max}(j) + j-1)2^{i^*} \right)}_{\text{transient phase, independent of $T$ }},\\
    &\text{ where }\\
    & i^*= \max\{8,i_1,i_2\},
    i_1 = \min\{i: (N-1)\tfrac{10\gamma i}{\Delta^2_{min}} < \beta 2^{(i-1)}\}, \text{ and }
    i_2 = \min \{i: (N-1 + NK(i-1))\leq 2^{i+1}\}.
\end{align*}
\end{theorem}

\section{Proof of Regret Upper Bound under $\boldsymbol{\alpha}$-Condition}\label{sec:app_alpha}
In this section we prove our main result for the instances satisfying $\alpha$-condition. We will present a short note on the main proof idea, while pointing out why the proof in the previous section does not go through. Next we present the necessary notations before going into the proof of the results. The proof structure, and some parts of the proof remain closely related to that of the previous section. Therefore, we mainly focus the new parts of the proof, while referring to the parts related to SPC we present proof sketch.

\subsection{Main Proof Idea} The key idea of the proof is similar to SPC condition but now before the global deletion starts to freeze, we need to talk about vanishing of local deletion for the stable matched arms (note the sub-optimal arms for each agent can still get locally deleted at this point). So the three important stages are: (1) local deletion vanishes for stable matched arms (from agent $A_1$ to $A_N$), (2) freezing of global deletion (from agent $1$ to $N$), (3) vanishing of local deletion of all arms (depending on when the blocking agents freeze global deletion).  We next elaborate more on why (1) should precede (2) under $\alpha$-condition whereas under SPC condition we can directly go to (2).

Under SPC for agent $1$ there was no risk of local deletion for it's stable match pair, which is also its best arm,  as for this arm agent $1$ is also the best agent. This sets up the inductive freezing of the global deletion as agent $1$ quickly identifies arm $1$ as it's best arm. The vanishing of local deletion is the consequence of the freezing of global deletion of the blocking agents. But under $\alpha$-condition it is no longer the case as agent $1$ is not the most preferred agent for arm $1$. Instead we have that the agent $A_1$ has no risk of local deletion of its stable match pair, $a_1$, which is (possibly) not the best arm for agent $A_1$ but for arm $a_1$ we have $A_1$ as its best agent. Therefore, agent $A_1$ will not delete it's stable match pair arm $a_1$, but unless global deletion eliminates better arms it will not converge to this arm. However, $A_1$ will stop causing local deletion (which we will prove) for the stable matched arm for agents in the set $\{j: A_1 >_{k^*_j} j, j \in [N]\}$. This will continue inductively. In particular, $A_1$ stops local deletion of stable matched arm of agent $A_2$ which in turn stops local deletion caused by agent $A_2$, so on and so forth. 

{\bf Where proof of SPC fails for $\boldsymbol{\alpha}$-condition?} Before going into the proof of $\alpha$-condition we identify why the proof in previous section fails. The key step that breaks when we move from SPC to $\alpha$-condition is that in Lemma~\ref{lemm:bound}  the ineuality~\eqref{eq:spckey} does not hold anymore. The issue is we do not have $k^*_j$ to be dominated only by the agents $1$ to $(j-1)$, i.e. there may exist agent $j'>j$ such that $j'>_{k^*_j} j$.  Similar idea is also exploited in Lemma~\ref{lemm:goodphase} which also fails to hold for the same reason.

\subsection{Notations and Definitions:} We setup the notations required for the regret upper bound proof when the system satisfies $\alpha$-condition.  The right-order in the definition of the $\alpha$-condition be given as $[N]_r=\{A_1,A_2,\dots, A_N\}$ (a permutation of $[N]$) for the agents, and  $\{a_1,a_1,\dots, a_K\}$ (a permutation of $[K]$) for the arms. Whereas, the left-order in the definition is $[N]$ and $[K]$. Also, we recall that $k^*_j$ as the stable matched arm for any agent $j\in [N]$, and $j^*_k$ as the stable matched agent for the arm $k$, for all $k \in [K], k \leq N$.

We now recall that due to $\alpha$-condition the following statements hold
\begin{align*}
    (i)\,&\forall j\in [N], \forall k > j \in [K],  \mu_{j j} >  \mu_{j k},\\
    (ii)\,&\forall a_k \in [K]_r, k \leq N, \forall j > k, A_j \in [N]_r, A_{j^*_{a_k}} >_{a_k} A_j,\\
    (iii)\,&\forall A_j\in [N]_r, k^*_{A_j} = a_j \in [K]_r,\\
    (iv)\,&\forall j\in [N], k^*_{j} = j \in [K],
\end{align*}
Here, (i) and (ii) follows from the definition of $\alpha$-stability and (iii) and (iv) follows from the Proposition~\ref{prop:stable}. Let us denote by $lr$ the mapping of agents in left order to agents in right order under $\alpha$-condition, i.e. agent $j = A_{lr(j)}$ for all $j\in [N]$.

\textbf{Arm Classification:} For each agent $j$, the {\em dominated arms}
($\mathcal{D}_j$),  the {\em blocking agents} for arm $k$ and agent $j$ ($\mathcal{B}_{jk}$), the set of {\em hidden arms } ($\mathcal{H}_j$) are defined identically to the SPC scenario. 
 Let $K_W(j)$ be the set of arms each of which is a stable matched arm for some other agent $j'$, is a sub-optimal arm for $j$, and $j$ is preferred by that arm than its stable pair $j'$, i.e. 
 $$K_W(j)=\{k: k\in [K], \mu_{jk} < \mu_{jk^*_j}, \exists j'\neq j: (k = k^*_{j'}, j >_k j')\}.$$
 We note that $K_W(A_j) \leq (K-j)$ as due to $\alpha$-condition agent $k^*_{j'}\notin  K_W(j)$ for any $j \leq N$.  

{\bf Ciritcal Phases:} We now define the critical phases when the system satisfies the $\alpha$-condition
\begin{itemize}
    \item The phase $i$ for agent $j$, for some $j \in [N]$, is a {\em Warmup Phase}  if the following are true  for each arm $k\in K_W(j)$,:
        \begin{itemize}
            \item[1.] in phase $i$ arm $k$ is matched with agent $j$ at most $\tfrac{10\alpha i}{\Delta_{jk}^2}$ times, 
            \item[2.] in phase $i$ arm $k$ is not agent $j$'s most matched arm
        \end{itemize}
    \item The phase $i$ for agent $j$, for some $j \in [N]$, is an {\em $\alpha$-Good Phase} if the following are true:
        \begin{itemize}
            \item[1.] The dominated arms are globally deleted, i.e. $G_j[i] = \mathcal{D}_j$.
            \item[2.] The phase $i$ is a {\em warmup phase} for all agents in $\mathcal{L}_j = \{j': k^*_j\in K_W(j')\}$.
            \item[3.] For each arm $k\notin \mathcal{D}_j \cup k^*_j$, in phase $i$ arm $k$ is matched with agent $j$ at most $\tfrac{10\alpha i}{\Delta_{jk}^2}$ times.  
            \item[4.] The stable match pair arm $k_j^*$ is matched the most number of times in phase $i$. 
        \end{itemize}
    The $\alpha$-good phase is not identical to good phase as condition (2) is additional in this case.
    
    \item A phase $i$ for agent $j$, for some $j \in [N]$ is called {\em $\alpha$-Low Collision Phase} if the following are true:
        \begin{itemize}
            \item[1.] Phase $i$ is a $\alpha$-good phase for agents $1$ to $j$.
            \item[2.] Phase $i$ is a $\alpha$-good phase for all agent $j' \in \cup_{k\in \mathcal{H}_j} \mathcal{B}_{jk}$.
        \end{itemize}
    The $\alpha$-low collision phase is identical to low collision phase (in SPC) except the good phase is replaced with $\alpha$-good phase.
\end{itemize}

We define for agent $j$, similar to SPC, $\mathbb{I}_{G_\alpha}[i,j]$ to be the indicator that phase $i$ is a $\alpha$-good phase, $\mathbb{I}_{LC_\alpha}[i,j]$ to be the indicator that phase $i$ is a $\alpha$-low collision phase,  and  
$\mathbb{I}_{W}[i,j]$ to be the indicator that phase $i$ is a warmup phase.

Let $i_1 = \min\{i: (N-1)\tfrac{10\gamma i}{\Delta^2_{min}} < \beta 2^{(i-1)}\}$.
For each agent $j$, the {\em $\alpha$-Freezing} ($F_{\alpha j}$) phase is the phase on or after which the agents $1$ to $(j-1)$ are in $\alpha$-good phase, and all the $j''\in \mathcal{L}_j$ (henceforth {\em deadlock agents}) are in warmup phase. 
$$
F_{\alpha j} = \max\left(i_1, \min \left(\{i: \prod_{i'\geq i}\left(\prod_{j'=1}^{(j-1)}\mathbb{I}_{G_{\alpha}}[i',j']\right)\left(\prod_{j''\in \mathcal{L}_j}\mathbb{I}_{W}[i',j'']\right) = 1 \} \cup \{\infty\}\right)\right).
$$
Also, we define  $\alpha$-Vanishing phase ($V_{\alpha j}$) similar to SPC 
$$V_{\alpha j} = \max\left(i_1, \min\left(\{i: \prod_{i'\geq i}\mathbb{I}_{LC_\alpha}[i',j] = 1\} \cup \{\infty\}\right)\right).$$Similar to SPC, $V_{\alpha j} = \max\left(F_{\alpha(j+1)}, \cup_{k\in \mathcal{H}_j}\cup_{j'\in \mathcal{B}_{jk}} F_{\alpha j'} \right)$ from the definition of low collision phase.

Finally, for each $j \leq N$, the phase $i$ is the {\em Unlocked} phase ($U_{j}$) if all phases on and after $i$ are warmup phases for all the agents $A_1$ to $A_{j}$. 
$$
U_j = \max\left(i_1, \min \left(\{i: \prod_{j'=1}^{lr(j)-1}\prod_{i'\geq i}\mathbb{I}_{W}[i',A_j'] = 1 \} \cup \{\infty\}\right)\right).
$$
This will be useful in quantifying the $\alpha$-freezing phase $F_{\alpha j}$ later on.

\subsection{Structural results for $\alpha$-condition}
In this section, we collect the important results that hold due to the combinatorial properties of the stable matching system that satisfies the $\alpha$-condition. 

\begin{proposition}\label{prop:alphastable}
If a system satisfies $\alpha$-condition then we have $k^*_j=j$, $j^*_{a_k} = A_k$ and $k^*_{A_j}= a_j$ for all $ 1\leq k,j \leq N$.
\end{proposition}
\begin{proof}
That under $\alpha$-condition $k^*_j=j$ for all $1\leq j \leq N$ follows identically to Proposition~\ref{prop:alphastable}. For the final relation we note that under $\alpha$-condition we have for $k=1$ we have 
 $A_{j^*_1} >_{a_1} A_j$ for all $j>1$. Thus $j^*_1 = A_1$. We can extend the same logic to obtain $j^*_{a_k} = A_k$ for all $ 1\leq k \leq N$.
\end{proof}

We now prove that the arm $k^*_j$ can be blocked only by agents in $\mathcal{L}_j$.
\begin{claim}\label{clm:stage1structural}
For a stable matching $\mathbf{k}^*$ and any agent $j$, we have $\{j': j' >_{k^*_{j}} j\} \subseteq  \mathcal{L}_j = \{j': k^*_j\in K_W(j')\}$.
\end{claim}
\begin{proof}
We have the stable matching $\mathbf{k}^*$. Let  $j >_{k^*_{j'}} j'$ and $\mu_{jk^*_{j}} < \mu_{jk^*_{j'}}$, then $(j, k^*_{j'})$ forms a blocking pair as arm $k^*_{j'}$ and agent $j$ will be both happier switching from their respective partners under $\mathbf{k}^*$. Therefore, $\mathbf{k}^*$ is not a stable matching. Thus, for a stable matching $\mathbf{k}^*$ and  any two agents $1\leq j, j'\leq N$, agent $j$ satisfies $\mu_{jk^*_{j}} > \mu_{jk^*_{j'}}$ if $j >_{k^*_{j'}} j'$. Thus, if $\{j' >_{k^*_{j}} j\}$ then $\mu_{jk^*_{j}} < \mu_{jk^*_{j'}}$ so $k^*_{j} \in K_W(j')$ so $j'\in \mathcal{L}_j$.
\end{proof}

We now characterize the set of deadlock agents for each agent $j$.
\begin{claim}\label{clm:deadlock}
For each agent $j\in [N]$, $\mathcal{L}_j \subseteq \{A_{j'}: j'=1,\dots, lr(j)-1\}$.
\end{claim}
\begin{proof}
From $\alpha$-condition we know that $\forall a_k \in [K]_r, k \leq N, \forall j > k, A_j \in [N]_r, A_{j^*_{a_k}} >_{a_k} A_j.$ Further, from Proposition~\ref{prop:alphastable} we know that $j^*_{a_k} = A_k$ for all $ 1\leq k \leq N$. Therefore, we can observe for any $j, j'\leq N$ and $j<j'$, $A_{j} >_{k^*_{A_j}} A_{j'}$. In particular, for any $j' > lr(j)$ we have $j = A_{lr(j)} >_{k^*_j} A_{j'}$.
Which means for any $j' \geq lr(j)$, we do not have $j'>_{k^*_j} j$ and hence  $k^*_j \notin K_W(j')$. This proves that for any $j' \geq lr(j)$  $j'\notin \mathcal{L}_j$, i.e. $\mathcal{L}_j \subseteq \{A_{j'}: j'=1,\dots, lr(j)-1\}$.
\end{proof}

We recall that $lr(j)$ is the index of the agent $j$ in the right-order of $\alpha$-condition.
The above characterization connects the unlock phase with the freezing phase as follows 
\begin{claim}\label{clm:fjupper}
For each agent $j\in [N]$, $F_{\alpha j} \leq \max\left(U_{(lr(j)-1)}, \max(F_{\alpha j'}: 1\leq j'\leq (j-1))\right)$ w.p. $1$. 
\end{claim}
\begin{proof}
Consider an arbitrary sample path.  We know by definition on or after phase $U_{(lr(j)-1)}$, all agents  $\{A_{j'}: j'=1,\dots, lr(j)-1\}$ are in warmup phase. We have the set of deadlock agents as $\mathcal{L}_j \subseteq \{A_{j'}: j'=1,\dots, lr(j)-1\}$. Hence, all agents in $\mathcal{L}_j$ are also in warmup phase on or after phase $U_{(lr(j)-1)}$. Further, the agents $1$ to $(j-1)$ are in $\alpha$-good phase from phase $\max(F_{\alpha j'}: 1\leq j'\leq (j-1))$ onwards. Hence, $F_{\alpha j} \leq \max\left(U_{(lr(j)-1)}, \max(F_{\alpha j'}: 1\leq j'\leq (j-1))\right)$ with probability~1.
\end{proof}

Next the following lemma captures a few key properties related to the critical phases. 
\begin{lemma}\label{lemm:alphastructural}
For $i \geq i_1 = \min\{i: (N-1)\tfrac{10\gamma i}{\Delta^2_{min}} < \beta 2^{(i-1)}\}$, any $j\in [N]$, 
\begin{itemize}
    \item if phase $i$ and $(i-1)$ are warmup phases for all $j'\in \mathcal{L}_j$  then $k^*_j \notin L_j[i] \cup G_j[i]$ almost surely,
    \item if phase $i \geq \min(U_{(lr(j) -1)}, F_{\alpha j}) + 1$ then $k^*_j \notin L_j[i] \cup G_j[i]$ almost surely,
    \item if phase $i \geq V_{\alpha j} + 1$ collision phase for agent $j$ then $L_j[i]=\emptyset$ almost surely.
\end{itemize}
\end{lemma}
\begin{proof}
The following results hold for an arbitrary sample path giving us almost sure inequalities.

Due to Claim~\ref{clm:stage1structural} all agents $j'$ which can block arm $k^*_j$ are in $\mathcal{L}_j$. Also $k^*_j \in K_W(j')$ for any agent $j'\in \mathcal{L}_j$ due to the definition of $\mathcal{L}_j$. Therefore, if all agents in $\mathcal{L}_j$ are in warmup phase in phase $(i-1)$ then $k^*_j \notin G_j[i]$ because no agent in  $\mathcal{L}_j$ communicates arm $k^*_j$ to agent $j$, and the other arms can not communicate the arm $k^*_j$ (due to this arm's preference). Furthermore, the total number of times the arm $k^*_j$ can be deleted is  at most $(lr(j)-1)\tfrac{10\alpha i}{\Delta_{jk}^2} < \beta 2^{(i-1)}$ (the local deletion threshold) for any $i \geq i_1$. Thus $k^*_j$ is not locally deleted, i.e.  $k^*_j \notin L_j[i]$.
This proves the first part. 

We know that the phase $i\geq U_{lr(j) -1 } + 1$ and $(i-1) \geq  U_{lr(j) -1 }$ is a {\em warmup phase} for all agents in $\mathcal{L}_j = \{j': k^*_j\in K_W(j')\}$. This is because we know that $\mathcal{L}_j \subseteq \{A_{j'}: j'=1,\dots, lr(j)-1\}$ due to Claim~\ref{clm:deadlock}. By definition of $F_{\alpha j}$ all agents are in warmup phase for phases $i\geq F_{\alpha j} + 1$ and $(i-1) \geq  F_{\alpha j}$. Thus the second result follows due to the first result. 

The proof of the third part follows almost identically to the Lemma~\ref{lemm:structural}, i.e. by virtue of $i\geq V_{\alpha j}+1$ being an $\alpha$-low collision phase.
\end{proof}

\subsection{Proof of main results}

In this section, we proceed with the regret bound where we leverage the structural properties proven in the previous part. We first state the regret decomposition lemma, which has an identical form to the regret decomposition as in SPC with $F_{\alpha j}$ and $V_{\alpha j}$ in place of $F_{j}$ and $V_{j}$, respectively. 
\begin{lemma}\label{lemm:alphadecompose}
The expected regret for agent $j$ can be upper bounded as  
\begin{align*}
    \mathbb{E}[R_j(T)] &\leq \mathbb{E}[S_{F_{\alpha j}}] + \min(\beta|\mathcal{H}_j|,1)\mathbb{E}[S_{V_{\alpha j}}] + \left((K-1+|\mathcal{B}_{jk^*_j}|)\log_2(T) + NK \mathbb{E}[V_{\alpha j}]\right)\\
    & + \sum_{k\notin \mathcal{D}_j} \sum_{j'\in \mathcal{B}_{jk}: k \notin \mathcal{D}_{j'}} \frac{8\gamma \mu_{k^*_j}}{\Delta^2_{j'k}}\left(\log(T)+\sqrt{\tfrac{\pi}{\gamma}\log(T)}\right)
    + \sum_{k\notin \mathcal{D}_j \cup k^*_j}\frac{8\gamma}{\Delta_{jk}}\left(\log(T)+\sqrt{\tfrac{\pi}{\gamma}\log(T)}\right) \\
    & + NK \left(1+(\psi(\gamma)+1)\frac{8\gamma}{\Delta^2_{\min}}\right)
\end{align*}
\end{lemma}
\begin{proof}[Proof Sketch]
The proof of the lemma is closely related to the proof of Lemma~\ref{lemm:decompose}, except for the use of the $\alpha$-freezing phase $F_{\alpha j}$ instead of the freezing phase $F_j$, and $\alpha$-vanishing phase $V_{\alpha j}$ instead of vanishing phaes $V_{j}$. The rest of the proof is identical to the proof of Lemma~\ref{lemm:decompose} where Lemma~\ref{lemm:alphabound} is invoked instead of it's identical counterpart (for SPC) Lemma\ref{lemm:bound}.
\end{proof}
\begin{lemma}\label{lemm:alphabound}
For any $j \in [N]$, $k\notin \mathcal{D}_j\cup k^*_j$, for $\gamma > 1$,
$$
\mathbb{E}\left[(N_{jk}(T) - N_{jk}(S_{F_j}))\right]
\leq \psi(\gamma)\tfrac{8}{\Delta_{jk}^2} + 1 + \tfrac{8}{\Delta_{jk}^2} \left(\gamma\log(T) + \sqrt{\pi\gamma\log(T)} + 1 \right). 
$$
\end{lemma}
\begin{proof}[Proof Sketch]
The proof of the lemma follows the proof of Lemma~\ref{lemm:bound}, again with $\alpha$-freezing phase $F_{\alpha j}$ in place of the freezing phase $F_j$. Due to Lemma~\ref{lemm:alphastructural} we know that for each phase $i\geq (F_{\alpha j}+1)$ the arm $k^*_j$ is available as it is neither globally deleted, nor locally deleted. Thus once a sub-optimal arm $k$ is played enough times the UCB of arm $k^*_j$ w.h.p. will be higher than the UCB of $k$ at any round after $F_{\alpha j}$. Using the same standard framework as in Lemma~\ref{lemm:bound} this intuition can be formalized as a proof of this lemma. 
\end{proof}

We first show that for phases $i\geq U_{j-1}+1$, the probability that phase $i$ is not a warmup phase for agent $A_j$ is low.
\begin{lemma}\label{lemm:warmup}
For any $j\leq N$ and any phase $i \geq i^* = \max(8, i_1, i_2)$ and $\gamma > 1$, 
$$\mathbb{P}[\mathbb{I}_{W}[i,A_j] = 0 \wedge i\geq U_{j-1}+1]\leq (K-j)2^{-i(\gamma -1)} \left( 1 + \tfrac{64}{\Delta_{\min}^2}\right),$$
where $i_1 = \min\{i: (N-1)\tfrac{10\gamma i}{\Delta^2_{min}} < \beta 2^{(i-1)}\}$ and $i_2 = \min \{i: (R-1 + C(i-1))\leq 2^{i+1}\}$.
\end{lemma}
\begin{proof}
For any arbitrary sample path and any $i\geq U_{j-1}+1$, phase $i$ is a warm up phase for all agent $A_1$ to $A_{j-1}$. 
The phase $i$ is not a warmup phase for agent $A_j$, if there exists an arm $k\in K_W(A_j)$ which is played more than $\tfrac{10 \gamma i}{\Delta^2_{A_jk}}$ times in phase $i$. Here, by definition for any $k\in K_W(A_j)$ we have $\mu_{A_jk} \leq \mu_{A_j a_j}$ (recall, $k^*_{A_j} = a_j$ due to Proposition~\ref{prop:alphastable}) which makes sure $\Delta_{A_jk}>0$. 

The set of agents that can block $A_j$ from matching with arm $a_j$ when $A_j$ plays $a_j$ is given by $\mathcal{L}_{A_j} \subseteq \{A_{j'}: 1\leq j'\leq j-1\}$ due to Claim~\ref{clm:deadlock} and $lr(A_j) = j$. But then due to the second point in Lemma~\ref{lemm:alphastructural} we know that $k^*_{A_j} \notin G_{A_j}[i]\cup L_{A_j}[i]$ for any $i\geq U_{j-1}+1$. Therefore, the inequality (i) below holds as to play arm $k$ the UCB of arm $k^*_j=a_j$ can not be less than arm $k$. The final bound can be obtained identically to the proof of Lemma~\ref{lemm:goodphase} for $i\geq \max(8, i_1, i_2)$, with the observation $K_W(j) \leq (K-j)$. 

Therefore, we obtain the next set of equations
\begin{align*}
    &\mathbb{P}[\mathbb{I}_W[i,A_j] = 0 \wedge i\geq (U_{j-1}+1)]\\
    &\leq \mathbb{P}\left[ \cup_{k\in K_W(A_j)} \{(N_{A_jk}[i] - N_{A_jk}[i-1]) > \tfrac{10\gamma i}{\Delta_{A_jk}^2} \} \wedge i\geq  (U_{j-1}+1)\right]\\
    &\leq \sum_{k\in K_W(A_j)}\mathbb{P}\left[ \cup_{t\in S_{i}}^{(S_{i+1}-1)} 
    N_{A_jk}(t) = \tfrac{10\gamma i}{\Delta_{A_jk}^2} \wedge I_{A_j}(t) = k\wedge i\geq  (U_{j-1}+1)\right]\\
    &\stackrel{(i)}{\leq}  \sum_{k\in K_W(A_j)} \sum_{t\in S_{i}}^{(S_{i+1}-1)} \mathbb{P}\left[ N_{A_jk}(t) = \tfrac{10\gamma i}{\Delta_{A_jk}^2} \wedge u_{A_jk}(t-1) > u_{A_ja_j}(t-1)\right]\\
     &\leq |K_W(A_j)|2^{-i(\gamma -1)} \left( 1 + \tfrac{64}{\Delta_{\min}^2}\right)\\
    &\leq (K-j)2^{-i(\gamma -1)} \left( 1 + \tfrac{64}{\Delta_{\min}^2}\right)
\end{align*}
This completes the proof.
\end{proof}

The proof of this lemma resembles closely that of  Lemma~\ref{lemm:goodphase} while some arguments are common to Lemma~\ref{lemm:warmup}.
\begin{lemma}\label{lemm:alphagoodphase}
For any agent $j$ and any phase $i \geq i^* =  \max\{8, i_1, i_2\}$ and $\gamma > 1$, 
$$\mathbb{P}[\mathbb{I}_{G_\alpha}[i,j] = 0\wedge i\geq F_{\alpha j}+1]\leq (K - j)2^{-i(\gamma -1)} \left( 1 + \tfrac{64}{\Delta_{\min}^2}\right),$$
\end{lemma}
where $i_1$ and $i_2$ is as defined in Lemma~\ref{lemm:warmup}. 
\begin{proof}
The phase $i$ is a $\alpha$-good phase for agent $j$ if (1) the dominated arms are deleted $G_j[i] = \mathcal{D}_j$, (2) phase $i$ is a {\em warmup phase} for all agents in $\mathcal{L}_j = \{j': k^*_j\in K_W(j')\}$, (3) for each arm $k\notin \mathcal{D}_j \cup k^*_j$, in phase $i$ arm $k$ is matched with agent $j$ at most $\tfrac{10\alpha i}{\Delta_{jk}^2}$ times, and (4) the stable match pair arm $k_j^*$ is matched the most number of times in phase $i$. 
We see that (1) and (2) holds when $i\geq F_{\alpha j}+1$. Also, (4) holds when (1), (2) and (3) holds for any $i\geq i_1$.

Therefore, we will now show (3) holds. In particular, we have the following series of inequalities
\begin{align*}
    &\mathbb{P}[\mathbb{I}_{G_\alpha}[i,j] = 0\wedge i\geq (F_{\alpha j}+1)]\\
    &\leq \mathbb{P}\left[ \cup_{k\notin  \mathcal{D}_j \cup k^*_j} \{(N_{jk}[i] - N_{jk}[i-1]) > \tfrac{10\gamma i}{\Delta_{jk}^2} \} \wedge i\geq (F_{\alpha j}+1)\right]\\
    &\leq \sum_{k\notin  \mathcal{D}_j \cup k^*_j}\mathbb{P}\left[ \cup_{t\in S_{i}}^{(S_{i+1}-1)} N_{jk}(t) = \tfrac{10\gamma i}{\Delta_{jk}^2} \wedge I_{j}(t) = k\wedge i\geq (F_{\alpha j}+1)\right]\\
    &\stackrel{(i)}{\leq}  \sum_{k\notin  \mathcal{D}_j \cup k^*_j} \sum_{t\in S_{i}}^{(S_{i+1}-1)} \mathbb{P}\left[ N_{jk}(t) = \tfrac{10\gamma i}{\Delta_{jk}^2} \wedge u_{jk}(t-1) > u_{jk^*_j}(t-1)\right]\\
    &\leq (K - j)2^{-i(\gamma -1)} \left( 1 + \tfrac{64}{\Delta_{\min}^2}\right).
\end{align*}
We know that for all arms $k \notin \mathcal{D}_j \cup k^*_j$ we have $\Delta_{jk} > 0$ by definition of $\mathcal{D}_j$. 
Also, inequality (i) holds as due to Lemma~\ref{lemm:alphastructural}, we know that after $i\geq (F_{\alpha j}+1)$ the arm $k^*_j$ is not globally or locally deleted.  The rest again follows similar to Lemma~\ref{lemm:goodphase} for $i \geq \max\{8, i_1, i_2\}$.
\end{proof}

Let us define $lr_{\max}(j) = \max(lr(j'): 1\leq j'\leq j)$, and $\tilde{F}_j = \max\left(U_{(lr_{\max}(j)-1)}, \max(\tilde{F}_{j'}: 1\leq j'\leq (j-1))\right)$ for each $j$. It is easy to see that $\tilde{F}_j > F_{\alpha j}$ due to Claim~\ref{clm:fjupper} for any $j$ and the fact that 
$U_{(lr_{\max}(j)-1)} \geq U_{(lr(j)-1)}$ due to the definition of $U_{j}$ (all agents from $A_1$ to $A_{j}$ all are in warmup phase till the end). We now present the the following lemma that bounds the probability that a phase $i$ is not an $\alpha$-good phase when 
$i \geq F_{j}+1$. We now bound the moments and exponents of $\tilde{F}_j$.

\begin{lemma}\label{lemm:alphamoments}
For any $j\in [N]$ and $m\geq 1$, the following hold with $i^*$ as defined in Lemma~\ref{lemm:alphagoodphase}
\begin{align*}
    &\mathbb{E}[\tilde{F}_j^m] \leq 2i_1 + (lr_{\max}(j)+j-2)\left((i^*)^m + K\left( 1 + \tfrac{64}{\Delta_{\min}^2}\right)  \tfrac{2^{- (\gamma -1)(i^*-2)}}{(2^{(\gamma-1)}-1)^2} \right)\\
    & \mathbb{E}[2^{\tilde{F}_j}] \leq 2i_1 + (lr_{\max}(j)+j-2)\left(2^{i^*} + K\left( 1 + \tfrac{64}{\Delta_{\min}^2}\right)  \tfrac{2^{- (\gamma -1)(i^*-2)}}{(2^{(\gamma-1)}-1)^2} \right)
\end{align*}
\end{lemma}
\begin{proof}
We again inductively bound the expectation of an arbitrary monotonically increasing and continuous (hence invertible)  function $g:\mathbb{R} \to \mathbb{R}_+$. We have that $F_0 = i_1$ almost surely by definition (this accounts for the max with $i_1$ in the definition of $F_j$). 

We calculate the expectation for agent $j$ as  
\begin{align*}
    \mathbb{E}[g(\tilde{F}_j)] &= \sum_{x\geq 0} \mathbb{P}[g(\tilde{F}_j) \geq x] \leq \sum_{x\geq 0} \mathbb{P}[\tilde{F}_j \geq g^{-1}(x)]\\
    &\leq \sum_{x\geq 0} \mathbb{P}[\tilde{F}_j \geq g^{-1}(x), U_{(lr_{\max}(j)-1)} \geq g^{-1}(x)] + 
    \sum_{x\geq 0} \mathbb{P}[\tilde{F}_j \geq g^{-1}(x), U_{(lr_{\max}(j)-1)} < g^{-1}(x)]\\
    &\leq \sum_{x\geq 0} \mathbb{P}[U_{(lr_{\max}(j)-1)} \geq g^{-1}(x)] + 
    \sum_{x\geq 0} \mathbb{P}[\tilde{F}_j \geq g^{-1}(x), U_{(lr_{\max}(j)-1)} < g^{-1}(x)]\\
    &\leq \mathbb{E}[g(U_{(lr_{\max}(j)-1)})] + 
    \sum_{x\geq 0} \mathbb{P}[\tilde{F}_j \geq g^{-1}(x), U_{(lr_{\max}(j)-1)} < g^{-1}(x)]\\
    &\leq i_1 + (lr_{\max}(j)-1)g(i^*) +  (lr_{\max}(j)-1)(K  - lr_{\max}(j)/2) \left( 1 + \tfrac{64}{\Delta_{\min}^2}\right)  \tfrac{2^{- (\gamma -1)(i^*-2)}}{(2^{(\gamma-1)}-1)^2}\\
    & + i_1 + (j-1)g(i^*) +  (j-1)(K  - j/2) \left( 1 + \tfrac{64}{\Delta_{\min}^2}\right)  \tfrac{2^{- (\gamma -1)(i^*-2)}}{(2^{(\gamma-1)}-1)^2} \\
    &\leq 2i_1 + (lr_{\max}(j)+j-2)g(i^*) \\
    &+ \left((lr_{\max}(j)+j-2)K  - (lr_{\max}(j)(lr_{\max}(j)-1)+j(j-1))/2\right) \left( 1 + \tfrac{64}{\Delta_{\min}^2}\right)  \tfrac{2^{- (\gamma -1)(i^*-2)}}{(2^{(\gamma-1)}-1)^2}\\
    &\leq 2i_1 + (lr_{\max}(j)+j-2)\left(g(i^*) + K\left( 1 + \tfrac{64}{\Delta_{\min}^2}\right)  \tfrac{2^{- (\gamma -1)(i^*-2)}}{(2^{(\gamma-1)}-1)^2} \right).
\end{align*}
The last inequality is loose, and we use it for simplicity. For the second last inequality  we use the following bounds on $\mathbb{E}[g(U_{(lr_{\max}(j)-1)})]$ and $\mathbb{P}[\tilde{F}_j \geq g^{-1}(x), U_{(lr_{\max}(j)-1)} < g^{-1}(x)]$ which we will prove momentarily.
\begin{align*}
&\mathbb{E}[g(U_{j})] \leq  i_1 + (j-1)g(i^*) +  (j-1)(K  - j/2) \left( 1 + \tfrac{64}{\Delta_{\min}^2}\right)  \tfrac{2^{- (\gamma -1)(i^*-2)}}{(2^{(\gamma-1)}-1)^2},\\
&\sum_{x\geq 0}\mathbb{P}[\tilde{F}_j \geq g^{-1}(x), U_{(lr_{\max}(j)-1)} < g^{-1}(x)] 
\leq i_1 + (j-1) g(i^*)+ (j-1)(K  - j/2) \left( 1 + \tfrac{64}{\Delta_{\min}^2}\right)  \tfrac{2^{- (\gamma -1)(i^*-2)}}{(2^{(\gamma-1)}-1)^2}.
\end{align*}

{\bf Case 1:} The base case $U_0=i_1$ holds almost surely by definition. We have 
\begin{align*}
    \mathbb{E}[g(U_{j})] &= \sum_{x\geq 0} \mathbb{P}[g(U_{j}) \geq x]
    \leq \mathbb{E}[g(U_{j-1})] + g(i^*) + 
    \sum_{i\geq i^*} \mathbb{P}[U_j \geq i, U_{j-1} < i]\\
    &\leq \mathbb{E}[g(F_{j-1})] + g(i^*) + \sum_{i\geq i^*} \mathbb{P}[\{\exists i'\geq i, \mathbb{I}_W[i',j]=0\}, U_{j-1}+1 \leq i]\\
    &\stackrel{(i)}{\leq} i_1 + (j-1) g(i^*)+ (j-1)(K  - j/2) \left( 1 + \tfrac{64}{\Delta_{\min}^2}\right)  \tfrac{2^{- (\gamma -1)(i^*-2)}}{(2^{(\gamma-1)}-1)^2}.
\end{align*}
Here, for (i) we use the inequality in Lemma~\ref{lemm:warmup}, and take summations over $i$ (similar to Lemma~\ref{lemm:moments}).

{\bf Case 2:} We  again proceed inductively. For any $j \in [N]$, we introduce the notation $$\mathcal{F}_j := \sum_{x\geq 0}\mathbb{P}[\tilde{F}_j \geq g^{-1}(x), U_{(lr_{\max}(j)-1)} < g^{-1}(x)].$$ 
In the base case, as $\tilde{F}_0 = i_1$,  we have 
$\mathcal{F}_0 \leq i_1.$
Proceeding with the inductive approach 
\begin{align*}
    &\sum_{x\geq 0} \mathbb{P}[\tilde{F}_j \geq g^{-1}(x), U_{(lr_{\max}(j)-1)} < g^{-1}(x)] \\
    &\leq \sum_{x\geq 0} \mathbb{P}[\tilde{F}_j \geq g^{-1}(x), \tilde{F}_{j-1} \geq g^{-1}(x), U_{(lr_{\max}(j-1)-1)} < g^{-1}(x)] \\
    &+ \sum_{x\geq 0}\mathbb{P}[\tilde{F}_j \geq g^{-1}(x), \tilde{F}_{j-1} < g^{-1}(x), U_{(lr_{\max}(j-1)-1)} < g^{-1}(x)]\\
    &\stackrel{(i)}{\leq} \sum_{x\geq 0}\mathbb{P}[\tilde{F}_{j-1} \geq g^{-1}(x), U_{(lr_{\max}(j-1)-1)} < g^{-1}(x)] \\
    &+\sum_{x\geq 0} \mathbb{P}[\{\exists i'\geq g^{-1}(x), \mathbb{I}_{G_{\alpha}}[i',j]=0\}, F_{\alpha(j-1)} < g^{-1}(x)]\\
    &\leq \mathcal{F}_{j-1} +
    \sum_{x\geq 0} \mathbb{P}[\{\exists i'\geq g^{-1}(x), \mathbb{I}_{G_{\alpha}}[i',j]=0\}, F_{\alpha(j-1)} < g^{-1}(x)]\\
    &\leq \mathcal{F}_{j-1} + g(i^*) + \sum_{i \geq i^*} \mathbb{P}[\{\exists i'\geq i,  \mathbb{I}_{G_{\alpha}}[i',j]=0\}, F_{\alpha(j-1)} < i]\\
    &\leq \mathcal{F}_{j-1} + g(i^*) + \sum_{i \geq i^*} \sum_{i'\geq i} \mathbb{P}[\{\exists i'\geq i,  \mathbb{I}_{G_{\alpha}}[i',j]=0\}, i \geq F_{\alpha(j-1)} + 1]\\
    &\leq \mathcal{F}_{j-1} + g(i^*) + \sum_{i'\geq i^*} (i'-i^*+1)\mathbb{P}[\{\exists i'\geq i^*,  \mathbb{I}_{G_{\alpha}}[i',j]=0\}, i' \geq F_{\alpha(j-1)} + 1]\\
    &\leq \mathcal{F}_{j-1} + g(i^*) + \sum_{i'\geq i^*} (i'-i^*+1)\mathbb{P}[\{\exists i'\geq i^*,  \mathbb{I}_{G_{\alpha}}[i',j]=0\}, i' \geq F_{\alpha(j-1)} + 1]\\
    &\stackrel{(ii)}{\leq}\mathcal{F}_{j-1} +  g(i^*)  +  (K - j) \left( 1 + \tfrac{64}{\Delta_{\min}^2}\right) \sum_{i'\geq i^*} (i'-i^*+1) 2^{-i'(\gamma -1)}\\
    &\leq i_1 + (j-1) g(i^*)+ (j-1)(K  - j/2) \left( 1 + \tfrac{64}{\Delta_{\min}^2}\right)  \tfrac{2^{- (\gamma -1)(i^*-2)}}{(2^{(\gamma-1)}-1)^2}.
\end{align*}
For the inequality (i) we use the fact that given $U_{(lr_{\max}(j-1)-1)}, \tilde{F}_{j-1} < g^{-1}(x)$ the only way we can have $\tilde{F}_{j-1} \geq g^{-1}(x)$ if for some phase $i' > g^{-1}(x)$ agent  $j$ is not in an $\alpha$-good phase. Then we use Lemma~\ref{lemm:alphagoodphase} to obtain inequality (ii).
\end{proof}


For the expected rounds upto the end of phase $F_j$ is upper bounded as 
\begin{align*}
&\mathbb{E}[S_{F_{\alpha j}}] = \mathbb{E}[ R + C(F_{\alpha j}-1) + 2^{F_{\alpha j}}] \leq \mathbb{E}[ R + C(\tilde{F}_j-1) + 2^{\tilde{F}_j}]\\
&\leq R + C(2i_1-1) + C(lr_{\max}(j)+j-2) i^* + (lr_{\max}(j)+j-2)2^{i^*} \\
&+ (C+1)(lr_{\max}(j)+j-2) K \left( 1 + \tfrac{64}{\Delta_{\min}^2}\right)  \tfrac{2^{- (\gamma -1)(i^*-2)}}{(2^{(\gamma-1)}-1)^2}
\end{align*}

Similar to SPC condition, we define $J_{\max}(j) = \max\left(j+1, \{j': \exists k\in \mathcal{H}_j, j'\in \mathcal{B}_{jk}\}\right)$.
Then as $\tilde{F}_{j} \geq \tilde{F}_{\alpha j'}$ almost surely for all $j \geq j'$ by definition and $\tilde{F}_j \geq F_{\alpha j}$, we have  
$$V_{\alpha j} = \max\left(F_{\alpha(j+1)}, \cup_{k\in \mathcal{H}_j}\cup_{j'\in \mathcal{B}_{jk}} F_{\alpha j'} \right) \leq \max\left(\tilde{F}_{(j+1)}, \cup_{k\in \mathcal{H}_j}\cup_{j'\in \mathcal{B}_{jk}} \tilde{F}_{j'} \right)  = \tilde{F}_{J_{\max}(j)}.$$ 

The regret upto the end of the phase when the local deletion vanishes is bounded as 
\begin{align*}
&\mathbb{E}[S_{V_{\alpha j}}] \leq \mathbb{E}[V_{\alpha j}^{(\beta+1)}] \leq \mathbb{E}[S_{\tilde{F}_{J_{\max}(j)}}]
\end{align*}

The  regret bound for the $\alpha$-condition in Theorem~\ref{thm:main} (identically derived as in the SPC case) is obtained by combining the above results as,
\begin{align*}
    &\mathbb{E}[R_j(T)]\\
    &\leq \sum_{k\notin \mathcal{D}_j} \sum_{j'\in \mathcal{B}_{jk}: k \notin \mathcal{D}_{j'}} \tfrac{8\gamma \mu_{k^*_j}}{\Delta^2_{j'k}}\left(\log(T)+\sqrt{\tfrac{\pi}{\gamma}\log(T)}\right) + \sum_{k\notin \mathcal{D}_j \cup k^*_j}\tfrac{8\gamma}{\Delta_{jk}}\left(\log(T)+\sqrt{\tfrac{\pi}{\gamma}\log(T)}\right) \\
    &+ c_j\log_2(T)  + O\left(\tfrac{N^2K^2}{\Delta^2_{\min}} + (\min(1, \beta |\mathcal{H}|_j)f_{\alpha}(J_{\max}(j)) +  f_{\alpha}(j)-1)2^{i^*}  + N^2K i^*\right)
\end{align*}
with the definition that $f_{\alpha}(j) = j + lr_{\max}(j)$.

This completes the proof of Theorem~\ref{thm:main}, as the regret bound for the SPC mentioned in the theorem holds due to Theorem~\ref{thm:spc}.


\section{Additional Experimental Results}\label{sec:app_experiments}
In this section, we present missing details of the dataset generation procedure and additional empirical results.

\subsection{Synthetic Dataset generation}
We use random instances to generate the results in this paper. For each instance the various algorithms are run for 50 times and the average and confidence intervals are constructed using these 50 trials. 

For the preference of the agents, we first create a random matrix $\mu \in [0, 1]^{N \times K}$ where each entry in the matrix is a i.i.d. $[0, 1]$ random variable. The minimum reward gap $\Delta_{min} \approx 0.05$ is enforced through rejection sampling. The agents preferences over the arms is given by the realization of this random matrix. We use different random matrices for different instances.

The preferences of the arms, varies across the three setting -- SPC, $\alpha$-condition, and general instances. 
\begin{itemize}
    \item For a general instance, we simply assign each arm with a random permutation over the agents as its preference list.
    \item  We start with a separate random preference list for each arm. To make this satisfy the SPC condition, we go in the order $1, 2, \ldots, K$ of the arms. For an arm $i$, we find  the first position in its preference assigned by the random permutation where an agent $j \geq i$ is present, then swap agent $i$ with agent $j$ to the end (if $j=i$ nothing is done). It is easy to see that this will satisfy the SPC condition.
    \item We generate the $\alpha$-condition instance by generating an arbitrary preference list (sample without replacement from possible permutations) for the arms, and then checking whether the instance (along with the agent preference fixed by the arm means) satisfies alpha condition following~\cite{karpov2019necessary}.  
\end{itemize}

For the UCB-D4 algorithm we use $\beta = 1/2K$, for the CA-UCB we use $\lambda = 0.2$ and for Phased ETC we use $\epsilon = 0.2$ for the $N=5$ and $K=6$ case.

\subsection{Performance of UCB-C, CA-UCB and UCB-D4 on general instances}
In this sub-section, we describe the results of the three algorithms with $N=5$ agents and $K=6$ arms on instances that go beyond the uniqueness consistency assumption. Note that in theory, UCB-C provides the optimal $\log(T)$ guarantee, CA-UCB provides a (possibly sub-optimal) guarantee of $\log^2(T)$ while we have no theoretical upper-bound on the regret of UCB-D4. Nonetheless, the results in Figure~\ref{fig:generalSmall} seem to indicate that CA-UCB has a potentially stronger theoretical upper-bound since its performance is very close to that of UCB-C which has $\log(T)$ upper-bound in the worst-case. Surprisingly, we also see that UCB-D4 \emph{converges} with all the agents eventually obtaining a sub-linear regret indicating that this algorithm may indeed have theoretical upper-bounds even in the more general setup. 

\begin{figure}[tb]
\centering
  \includegraphics[scale=0.25]{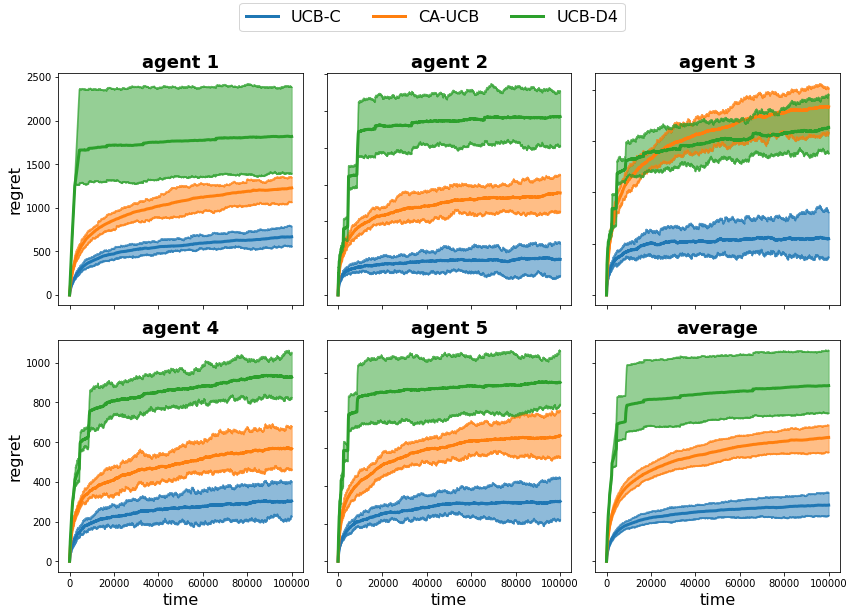}
\caption{Regret in a general instance (not satisfying $\alpha$-condition) with $5$ agents and $6$ arms.}
\label{fig:generalSmall}
\end{figure}

\begin{figure*}[ht!]
\centering
 \begin{subfigure}{0.5\textwidth}
\centering
  \includegraphics[width=0.85\textwidth]{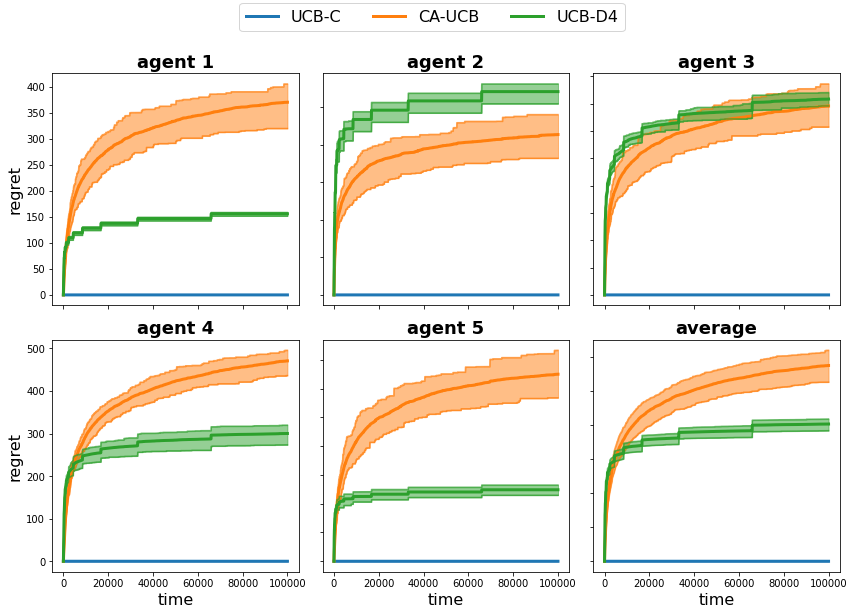}
\caption{Instance satisfying  SPC.}
\label{fig:spcCol}
\end{subfigure}%
\begin{subfigure}{0.5\textwidth}
\centering
  \includegraphics[width=0.85\textwidth]{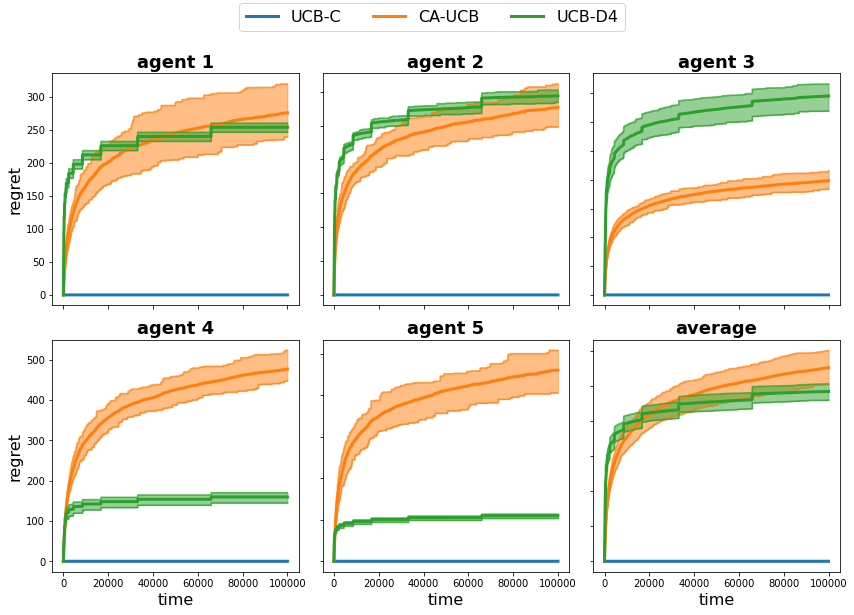}
\caption{Instance satisfying $\alpha$-condition.}
\label{fig:alphaCol}
\end{subfigure}

\begin{subfigure}{0.5\textwidth}
\centering
  \includegraphics[width=0.85\textwidth]{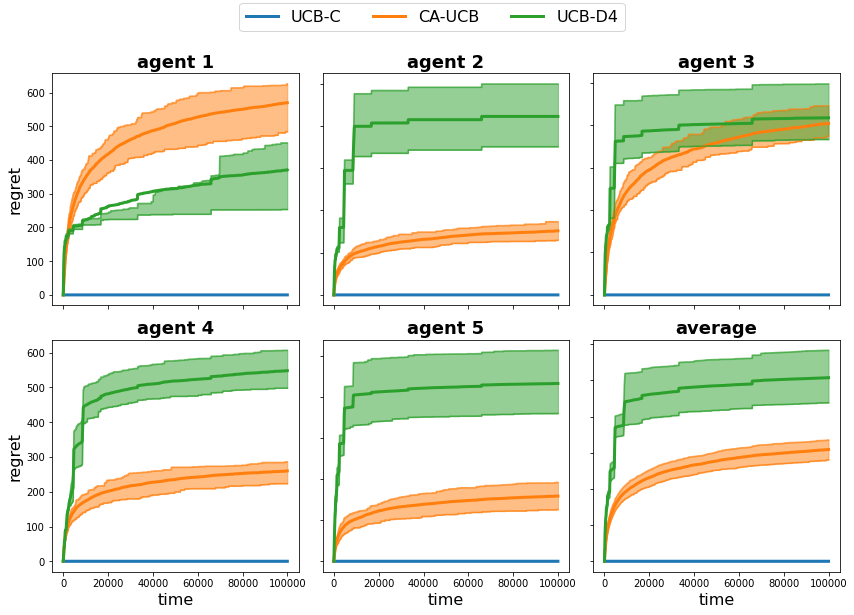}
\caption{General instances.}
\label{fig:generalCol}
\end{subfigure}
\caption{Collision regret comparison with $5$ agents and $6$ arms.}
\label{fig:Col}
\end{figure*}
\subsection{Collision Regret for UCB-C, CA-UCB and UCB-D4}

In this sub-section, we show the collision regret incurred by each of the three algorithms in the three settings under which we study their overall regret. As expected, UCB-C has no collision regret because of centralized communication. Surprisingly, CA-UCB has high regret due to collision despite having additional feedback in the SPC setting. This seems to indicate that most of the regret contribution for CA-UCB comes from collisions and once they are resolved the dynamics of should settle to a state which incurs no further regret.




\subsection{Performance of the algorithms on larger instances}
\label{subsec:largeInstances}

In this sub-section we run the algorithms for larger instances. In particular, we have $N=11$ agents and $K=15$ agents.\footnote{The number $11$ was chosen to obtain a rectangular $3 \times 4$ grid plot} 

{\em Tuned Phase Length:} We tune the phase length for larger instances. The tuning mainly balances some boundary conditions arising due to large communication blocks (which is only there in the fully decentralized setting) for large instances. Specifically, with large instances in the initial phases communication creates large regret if the phase lengths are small where not many samples can be explored.  For tuning Phased ETC (Algorithm~\ref{algo:etc}) we use exponent $c_0$, and multiplier $c_1$, where the  $i$-th phase now has length $c_1 \times c_0^{i}$. We have $c_1=1$ and $c_0=2$ for Algorithm~\ref{algo:etc}. For tuning UCB-D4 (Algorithm~\ref{alg:UCBD4}) we introduce exponent $c_0$, and multiplier $c_1$, where the  $i$-th phase now has length 
$\left((N-1)K + c_1 \times c_0^{i}\right)$. The UCB-D4(Algorithm~\ref{alg:UCBD4}) presented in the main paper we have $c_0=2$, and $c_1=1$.

The hyper-parameters for these plots are as follows. We use \\
1. phase exponent $c_0=1.5$, phase multiplier $c_1 = 1$, and exploration degree $\epsilon = 0.2$ for Phased-UCB,\\
2. phase exponent $c_0=1.2$, phase multiplier $c_1 = 3$, and the local collision threshold $\beta = 1/2K$ for UCB-D4, and \\
3. $\lambda = 0.2$ for CA-UCB.

The results that were previously observed also hold similarly for this larger instance. 
We note that the negative regret in the centralized UCB is natural, as during the initial phases an agent can match with an arm which has higher mean than its stable matched arm.

\begin{figure}[ht!]
\centering
  \includegraphics[width=0.7\textwidth]{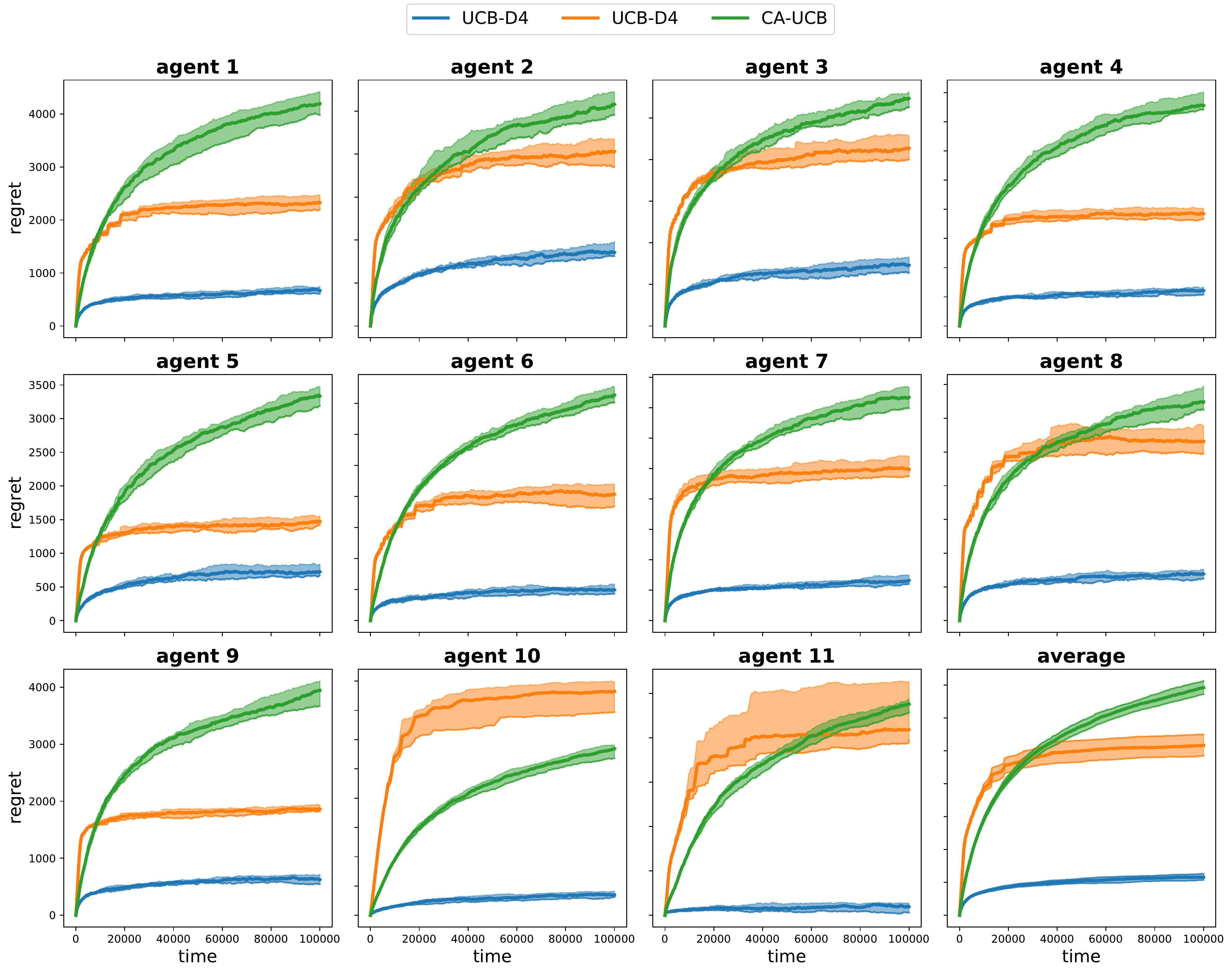}
\caption{Instance satisfying  SPC  with $11$ agents, and $15$ arms.}
\label{fig:spc_large}
\end{figure}

\begin{figure}[ht!]
\centering
  \includegraphics[width=0.7\textwidth]{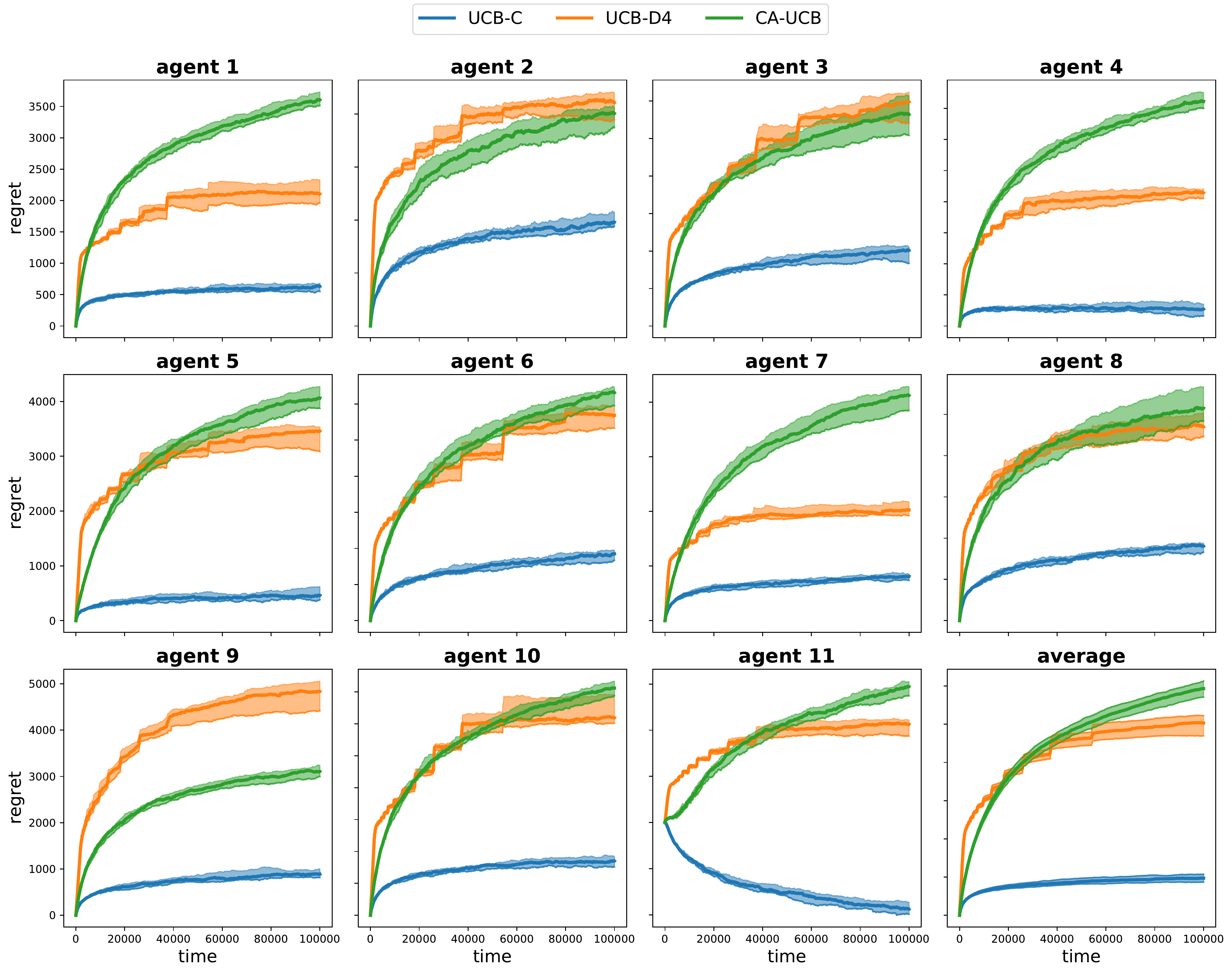}
\caption{Instance satisfying $\alpha$-condition  with $11$ agents, and $15$ arms.}
\label{fig:alpha_large}
\end{figure}

\begin{figure}[ht!]
\centering
  \includegraphics[width=0.7\textwidth]{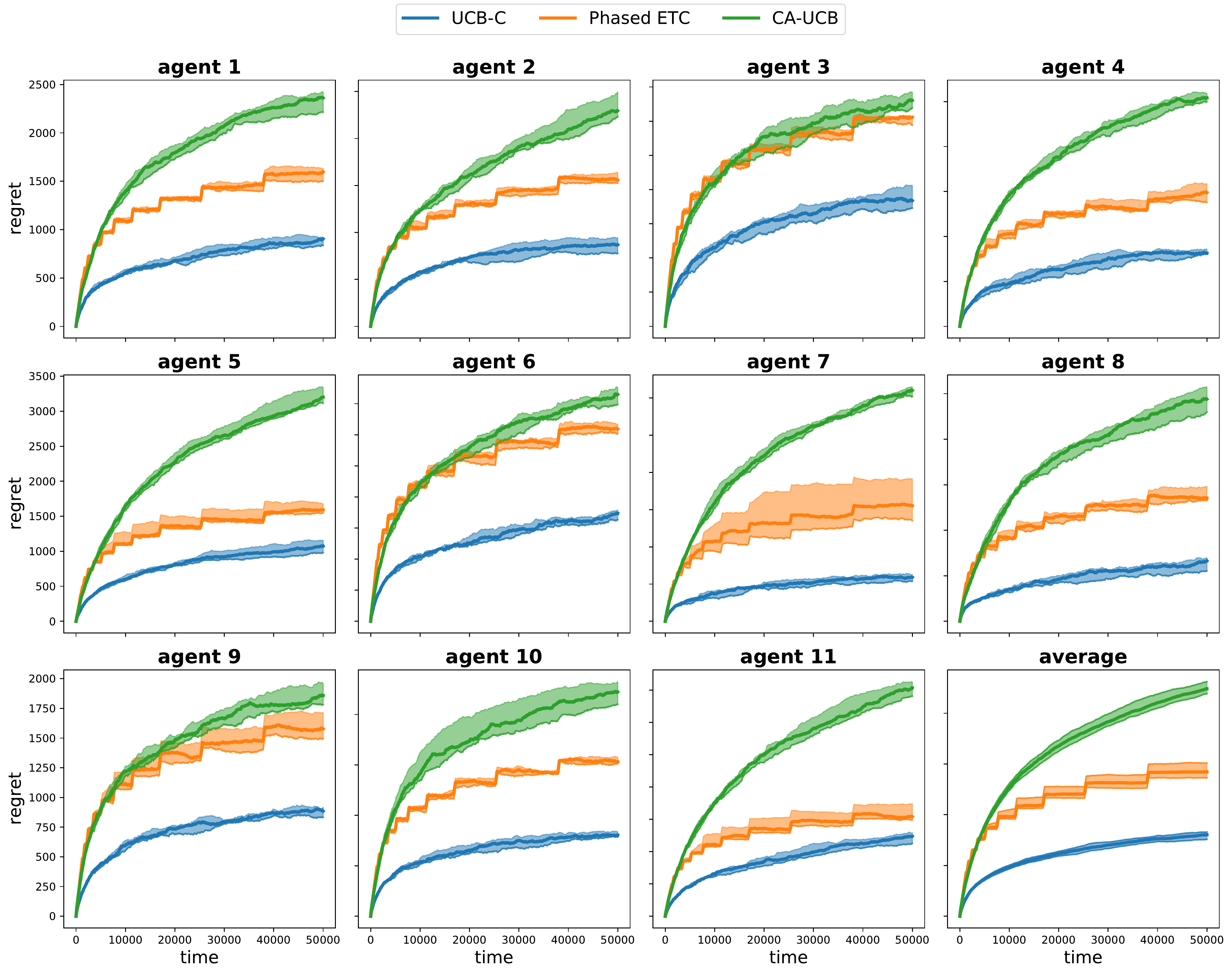}
\caption{General instances  with $11$ agents, and $15$ arms.}
\label{fig:general_large }
\end{figure}

{\em Regret Guarantees:} The regret bounds remain mostly unchanged due to the above tuning.  The regret of the modified Phased ETC is given by replacing the $\log_2(T)$ by $\log_{c_0}(T/c_1)$, and changing the constant to $\Theta\left(c_0^{ 1/\Delta^{2/\varepsilon}}\right)$.
For the modified UCB-D4 algorithm the $\log(T)$ regret due to collision and sub-optimal play does not change. The communication regret changes to $(K - 1 + |\mathcal{B}_{jk^*_j}|)\log_{c_0}(T/c_1)$. Finally, the constant part of the regret still remains $O\left(\max\left\{\tfrac{N}{\Delta^2_{min}}\log(\tfrac{N}{\Delta^2_{min}}), NK \log(NK)\right\}\right)$.
\end{appendices}

\end{document}